\documentclass{article}
\usepackage{spconf,amsmath,graphicx}

\usepackage{cite}
\usepackage{amssymb,amsfonts}
\usepackage{algorithmic}
\usepackage{graphicx}
\usepackage{textcomp}
\usepackage{amssymb,amsfonts,mathtools}
\usepackage[lofdepth,lotdepth]{subfig}
\usepackage{bm}
\usepackage{balance,cite}
\usepackage{color, colortbl}
\usepackage{hyperref}
\usepackage{cleveref}
\usepackage{float}
\usepackage{booktabs}
\usepackage{algorithm}
\usepackage{wrapfig}
\usepackage{algorithmic}
\usepackage{multirow}
\usepackage{verbatim}
\usepackage{amsthm}
\usepackage{bbm}

\makeatletter
\def\th@plain{%
	\thm@notefont{}
	\itshape 
}
\def\th@definition{%
	\thm@notefont{}
	\normalfont 
}
\makeatother

\makeatletter
\def\endthebibliography{%
	\def\@noitemerr{\@latex@warning{Empty `thebibliography' environment}}%
	\endlist
}
\makeatother
\DeclareMathOperator*{\argmin}{arg\,min}
\DeclareMathOperator*{\argmax}{arg\,max}
\newcommand{\rom}[1]{\uppercase\expandafter{\romannumeral #1\relax}}
\DeclarePairedDelimiterX{\norm}[1]{\lVert}{\rVert}{#1}
\DeclarePairedDelimiterX{\bnorm}[1]{\biggl\lVert}{\biggr\rVert}{#1}
\DeclarePairedDelimiterX{\abs}[1]{\lvert}{\rvert}{#1}

\renewcommand{\emph}[1]{{\textit{#1}}}

\newtheorem{definition}{Definition}
\theoremstyle{definition}

\newtheorem{theorem}{Theorem}

\newtheorem{assumption}{Assumption}
\newtheorem{proposition}{Proposition}

\def\R{\mathbb{R}}

\def\P{{ \mathrm{pr} }}
\def\v{{\varepsilon}}

\def\E{\mathbb{E}}

\def\X{\mathcal{X}}
\def\Y{\mathcal{Y}}

\def\P{\mathbb{P}}

\def\de{\overset{\Delta}{=}}

\def\de{\overset{\Delta}{=}}

\def\v{\varepsilon}

\def\P{\mathcal{P}}

\def\X{\mathcal{X}}
\def\Y{\mathcal{Y}}
\def\E{\mathbb{E}}

\makeatletter
\newcounter{savesection}
\newcounter{apdxsection}
\renewcommand\appendix{\par
  \setcounter{savesection}{\value{section}}%
  \setcounter{section}{\value{apdxsection}}%
  \setcounter{subsection}{0}%
  \gdef\thesection{\@Alph\c@section}}
\newcommand\unappendix{\par
  \setcounter{apdxsection}{\value{section}}%
  \setcounter{section}{\value{savesection}}%
  \setcounter{subsection}{0}%
  \gdef\thesection{\@arabic\c@section}}
\makeatother

\def\bbx{\mathbf{x}}

\def\fs{f_S} 
\def\fa{f_A} 
\def\bbx{X} 
\def\bby{Y} 
\def\testx{x} 
\def\testy{y} 
\def\ind{\mathbf{1}}
\def\sgn{\textrm{sgn}} 
\def\loss{\mathcal{L}} 
\def\ty{\tilde{y}} 
\def\testx{x}

\usepackage{balance}
\usepackage{xcolor}
\def\BibTeX{{\rm B\kern-.05em{\sc i\kern-.025em b}\kern-.08em
    T\kern-.1667em\lower.7ex\hbox{E}\kern-.125emX}}

\title{A Framework for Understanding Model Extraction Attack and Defense}
\name{Xun Xian$^{\star}$ \qquad Mingyi Hong$^{\star}$ \qquad Jie Ding$^{\dagger}$\thanks{This paper is based upon work supported by the National Science Foundation under grant number DMS-2134148.}}
\address{$^{\star}$ Department of Electrical and Computer Engineering, University of Minnesota \\
			    $^{\dagger}$ School of Statistics, University of Minnesota}

\begin{document}
\onecolumn
%
\maketitle
\begin{abstract}
The privacy of machine learning models has become a significant concern in many emerging Machine-Learning-as-a-Service applications, where prediction services based on well-trained models are offered to users via pay-per-query. The lack of a defense mechanism can impose a high risk on the privacy of the server's model since an adversary could efficiently steal the model by querying only a few `good' data points. The interplay between a server's defense and an adversary's attack inevitably leads to an arms race dilemma, as commonly seen in Adversarial Machine Learning. To study the fundamental tradeoffs between model utility from a benign user's view and privacy from an adversary's view, we develop new metrics to quantify such tradeoffs, analyze their theoretical properties, and develop an optimization problem to understand the optimal adversarial attack and defense strategies. The developed concepts and theory match the empirical findings on the `equilibrium' between privacy and utility. In terms of optimization, the key ingredient that enables our results is a unified representation of the attack-defense problem as a min-max bi-level problem. The developed results will be demonstrated by examples and experiments.
\end{abstract}

\section{Introduction}

An emerging concern of contemporary artificial intelligence applications is to protect the `privacy' of machine learning models. Machine learning models are proprietary in that they encompass private information extracted from the underlying data and domain-specific intelligence. For instance, it has been shown that querying spam to fraud detectors could lead to information leak~\cite{lowd2005adversarial,ateniese2015hacking,fredrikson2014privacy,fredrikson2015model}. Additionally, machine learning models are the valuable intellectual property of the providers. If an adversary can efficiently counterfeit the functionality of a well-trained model, the service providers may suffer from the economical loss. For example, in algorithmic trading, what matters the most is the algorithm being deployed instead of the public market data.

The concerns on the privacy of machine learning models are
particularly prominent in cloud-based large-scale learning services, e.g., Machine-Learning-as-a-Service (MLaaS)~\cite{alabbadi2011mobile,ribeiro2015mlaas} and multi-organizational learning~\cite{xian2020assisted}. In MLassS, for example, a cloud server often constructs a private supervised learning model based on its features-labels pairs. It then offers predictive service through an application programming interface (API), where it generates a prediction label for any future features.
Several recent works have shown that an adversary could extract a server's well-trained model through a prediction API~\cite{tramer2016stealing, chandrasekaran2018model, sethi2018data,chandrasekaran2018exploring,shi2017steal,shi2018active,orekondy2019knockoff}, also known as Model Extraction Attack. In these works, an adversary, who may or may not know the model architecture and the distributions of training data, aims to efficiently and effectively reconstruct a model close to the API's actual functionality based on a sequence of query-response pairs. 

Meanwhile, several recent works have developed promising strategies to defend or detect the model extraction attack~\cite{kesarwani2018model,juuti2019prada,orekondy2019prediction,zheng2019bdpl,IL}. 
Oftentimes, the above literature has been developed so that new model stealing attacks were proposed to break previous defenses, and further defenses are proposed to mitigate earlier attacks.
While it is exciting and inspiring to see the back-and-forth development of attack and defend strategies, 
there has not been a view that simultaneously considers the simultaneous interactions between attack and defense strategies.
In this work, we will develop such a unified framework to investigate the interplay between the server/defender and the adversary/attacker, where an adversary aims to craft queries sent to the server, and the server aims to identify the most effective way to maneuver the generated response.



\vspace{-0.2cm}
\subsection{Contributions}
\vspace{-0.2cm}

Our contributions of this work are summarized below.
\begin{itemize}
\vspace{-0.1cm}
    \item We formulate the notion of adversary-benignity (AB) curve to quantify the trade-offs between a benign user's utility and an adversary's stealing power. The AB curve is a general metric to evaluate the quality of a pair of defense and attack strategies. 
    
    \vspace{-0.1cm}
    \item We formulate a min-max bi-level framework that unifies the views of the user/attacker and server/defender. To the best of the authors' knowledge, this is the first optimization framework that unifies two kinds of strategies in the literature of adversarial model stealing attack. 
    
    \vspace{-0.1cm}
    \item We develop theoretical analyses of the proposed { notions}. Under some conditions, we present a (constructive) defense strategy that delivers information-theoretical guarantees. 
    
    \vspace{-0.1cm}
    \item We develop an operational algorithm to solve the proposed optimization framework. 
    Based on the results obtained from our algorithm, we discuss several practical implications on the model stealing attack and defense. 
    The key ingredient that enables our results is a representation of the problem as min-max bi-level optimization. The optimization formulation may be used in many other adversarial learning situations, so it has its own merit.
    \vspace{-0.1cm}
    
\end{itemize}

\vspace{-0.1cm}
\subsection{Related Work}
\vspace{-0.2cm}
\textbf{Data privacy} is a related literature that has received extensive attention in recent years due to ethical and societal concerns~\cite{evans2015biometrics,voigt2017eu}.
Data privacy concerns the protection of individual data information from different perspectives, such as cryptography~\cite{yao1982protocols,chaum1988multiparty}, randomized data collection~\cite{evfimievski2003limiting,kasiviswanathan2011can,DingInterval1a}, statistical database query~\cite{dwork2004privacy,dwork2011differential}, membership inference~\cite{shokri2017membership}, and Federated learning~\cite{shokri2015privacy,konevcny2016federated,mcmahan2017communication,DingHeteroFL}.
While the goal of data privacy is to obfuscate individual-level data values, the subject of model privacy focuses on protecting a single learned model ready to deploy. 
For example, we want to privatize a classifier to deploy on the cloud for public use, whether the model is previously trained from raw image data or a data-private procedure. 

\textbf{Model extraction} is another closely related subject~\cite{tramer2016stealing,papernot2016practical}, where a user's goal is to reconstruct a server's model from several query-response pairs, \textit{knowing the model architecture}. For example, suppose that the model is a generalized linear regression with $p$ features. It is easy to be reconstructed using $p$ queries of the expected mean (a known function of $X \beta$) by solving a linear equation system~\cite{tramer2016stealing}.  
When only hard-threshold labels are available, model extraction can be cast as an active learning problem where the goal is to query most efficiently \cite{chandrasekaran2018model}.
 
\textbf{Model defense strategies} have been recently studied from different perspectives. A warning-based method was developed in \cite{juuti2019prada}, where the server continuously test whether the pairwise distances among queried data approximately follow the Gaussian distribution empirically observed from benign queries. 
For classification models that return class probabilities \cite{lee2018defending}, a defense method was developed that maximally perturbs the probabilities under the constraint that the most-likely class label remains the same. 
The work in \cite{orekondy2019prediction} studied a similar setting but from a different view. The main idea was to perturb the probabilities within an $\ell_2$-distance constraint to poison the adversary's gradient signals.
A nonparametric method based on information theory was developed in \cite{IL}.



\section{Problem Formulation}\label{formu}
\vspace{-0.1cm}


We consider supervised learning, where $x \in \mathcal{X} \subset \R^d$, $y \in \mathcal{Y}$ denote the features and label, respectively.
Depending on the regression or classification learning, $\Y$ can be a subset of $\R$ or a finite set. We will make more specific assumptions later on. 
Let $f: \mathcal{X} \rightarrow \mathbb{R}$ denote the function that represents the supervision. With a slight abuse of notation, a function $f$ will be used to represent a learned model. For example, in regression models, $f(x)=\E(y\mid x)$, and 
in binary classification models, $f(x) = \log(\P(y=1\mid x)/\P(y=0 \mid x))$. 
We let $\loss$ denote the loss function that evaluates the discrepancy between the true label and predicted label. Examples are quadratic loss, $0-1$ loss, and cross-entropy loss.
Also, let $x_1,\ldots, x_n \in \R^d$ denote the query samples sent from the user to the server. For notational convenience, we will also let $\bbx_n$ denote the concatenation $[x_1,\ldots, x_n]\in \R^{nd}$. 
Let $\ind\{\cdot\}$ denote the $0-1$ indicator function.
Let $\sigma: u \mapsto (1+e^{-u})^{-1}$ denote the softmax function.




\subsection{Model Extraction and Defense}\label{med}

We now elaborate the process of model extraction and defense. For the ease of illustration, we will discuss in the context of binary classification.

\textbf{Server's point of view.} The server has trained a local model, denoted by $\fs: \mathcal{X} \rightarrow \mathbb{R} $. 
Upon receiving a user's querying samples $ \bbx_n \de [x_1, x_2, \ldots, x_n]$, the server will return $Y_n \de [y_1, y_2, \ldots, y_n] = [\fs(x_1), \fs(x_2), ..., \fs(x_n)]$, or its hard-threshold label version $\sgn(\fs(X_n))$ in some black-box model extraction scenarios.

Without knowing whether the user is benign or adversarial, the server will return $\fs^g(X_n)$ instead of $\fs(X_n)$ to the user to enhance model privacy.
The superscript $g$ denotes a defense strategy selected from the strategy set $\mathcal{G}$.
After selecting a defense strategy, the server will return labels $\widetilde{\bby}^{g}_n \de  [\fs^g(x_1),\ldots, \fs^g(x_n)],$ or its hard-threshold labels $\sgn(\widetilde{\bby}^{g}_n)$. When there is no ambiguity, we simply write $\widetilde{\bby}_n$ without explicitly specifying the defense strategy $g$. 



The only concern for a normal/benign user is whether it can obtain a low in-sample prediction error via the server's machine learning service. Hence, we define the server's utility (for benign users) in the following way.

\begin{definition}[Server's utility]\label{def_ser} 
Given an evaluation loss $\loss$ that is upper bounded by a positive constant $K$\footnote{The requirement of boundness is merely for technical considerations.}, the in-sample error of $\tilde{\bby}_n$ from $\bby_n$ is $n^{-1}\sum_{i=1}^n \loss(y_i,\tilde{y}_i)$.
The \textit{server's utility for benign users} is
$1 - K^{-1} n^{-1}\sum_{i=1}^n\loss(y_i,\tilde{y}_i).$
\end{definition}
From the definition, a server's utility is within $[0,1]$.
In particular, for the $0-1$ loss, $\loss(y,\tilde{y})\de \ind_{y \neq \tilde{y}}$, the {server's utility} is $1  - n^{-1}\sum_{i=1}^n\loss(y_i,\tilde{y}_i)$.


\textbf{Adversary's action}.
The adversarial user will send a set of $n$ queries to the server, which can be deterministic or randomly generated from a distribution $P$.
After obtaining labels $\tilde{\bby}_n$ returned from the server, the adversary will build a machine learning model $ \fa \in \mathcal{F}_A: \mathcal{X} \rightarrow \mathbb{R} $ to mimic the server's model $\fs$. The quality of the adversary's model is evaluated using the out-sample test error, where the test is based on the server's authentic responses, i.e., $\fs(\testx)$ for a test sample $\testx \in \X$.
Suppose that the test data $\testx \in \mathcal{X}$ follow a distribution $Q$, not necessarily the querying samples $X_n$'s distribution. A commonly used $Q$ (both in theory and practice) is the uniform distribution.
We define 
$$
R(\fa,\fs)= \mathbb{E}_{\testx \sim Q}\loss(\fa(\testx),\testy), \quad
\textrm{ where } \testy = \fs(\testx)
$$
as the out-sample prediction error of $\fa$ with respect to the server's model $\fs$.
Recall that the loss is assumed to be no larger than $K$, and consequently, the risk is also bounded by $K$. 
For ease of calculation, we will use the empirical version i.e., $R_n(\fa,\fs) = n^{-1} \sum_{i=1}^n \mathcal{L}(\fa(x_i),\fs(x_i))$. 
%
\vspace{0.2cm}
\begin{definition}[Adversary's utility] \label{def_adv} 
An adversary's utility with respect to a server $\fs$ is defined to be
$$1 - K^{-1} R_n(\fa,\fs) \in [0,1].$$
\end{definition} 



\subsection{Server's Utility and Adversary's Utility}
\begin{figure*}[!htb]
   \vspace{-0.2cm}
	\centering
	\subfloat[][]{\includegraphics[width=8.5cm, height = 5.3cm]{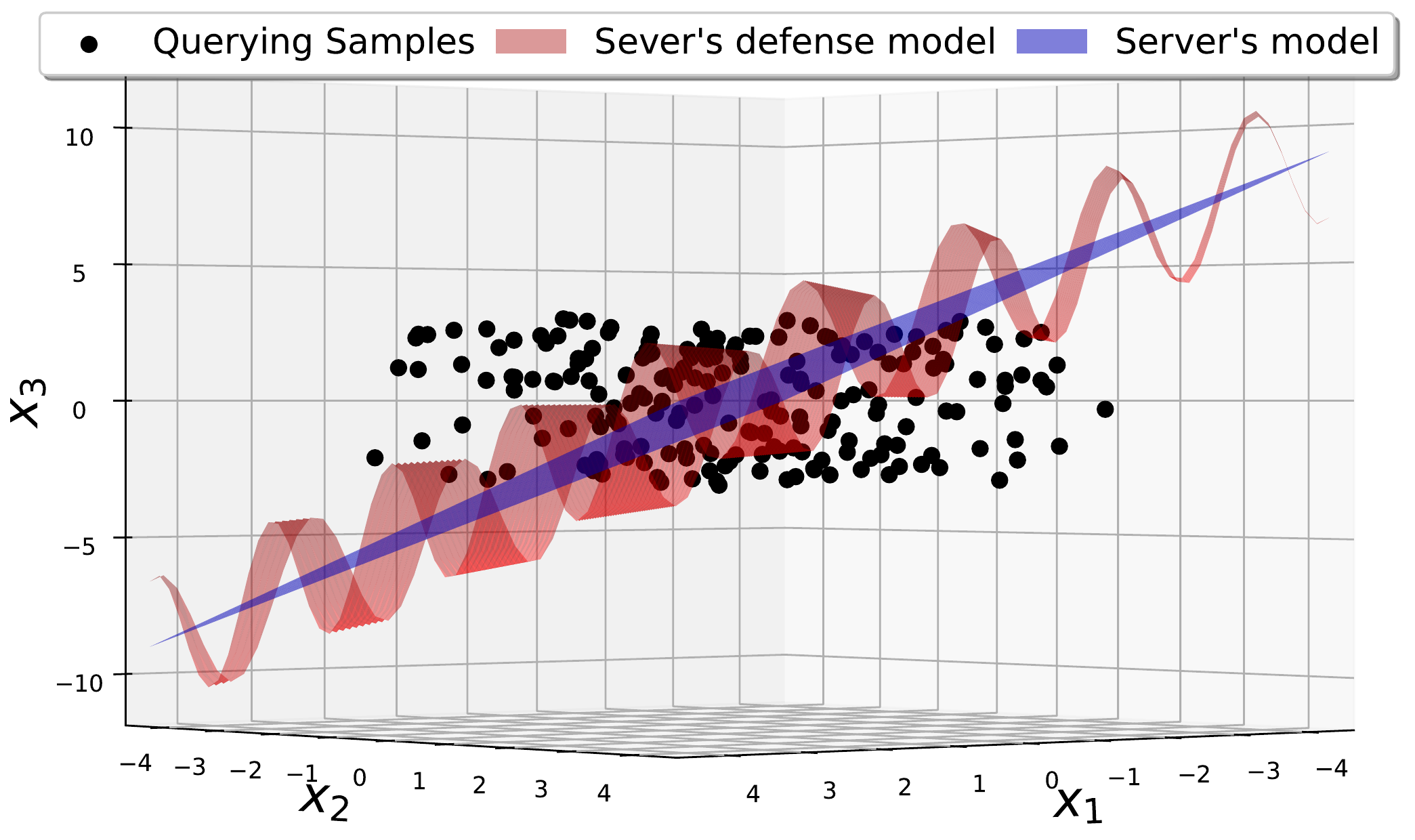}\label{fig_ill_case1}}\hfill
	\subfloat[][]{\includegraphics[width=8.5cm, height = 5.3cm]{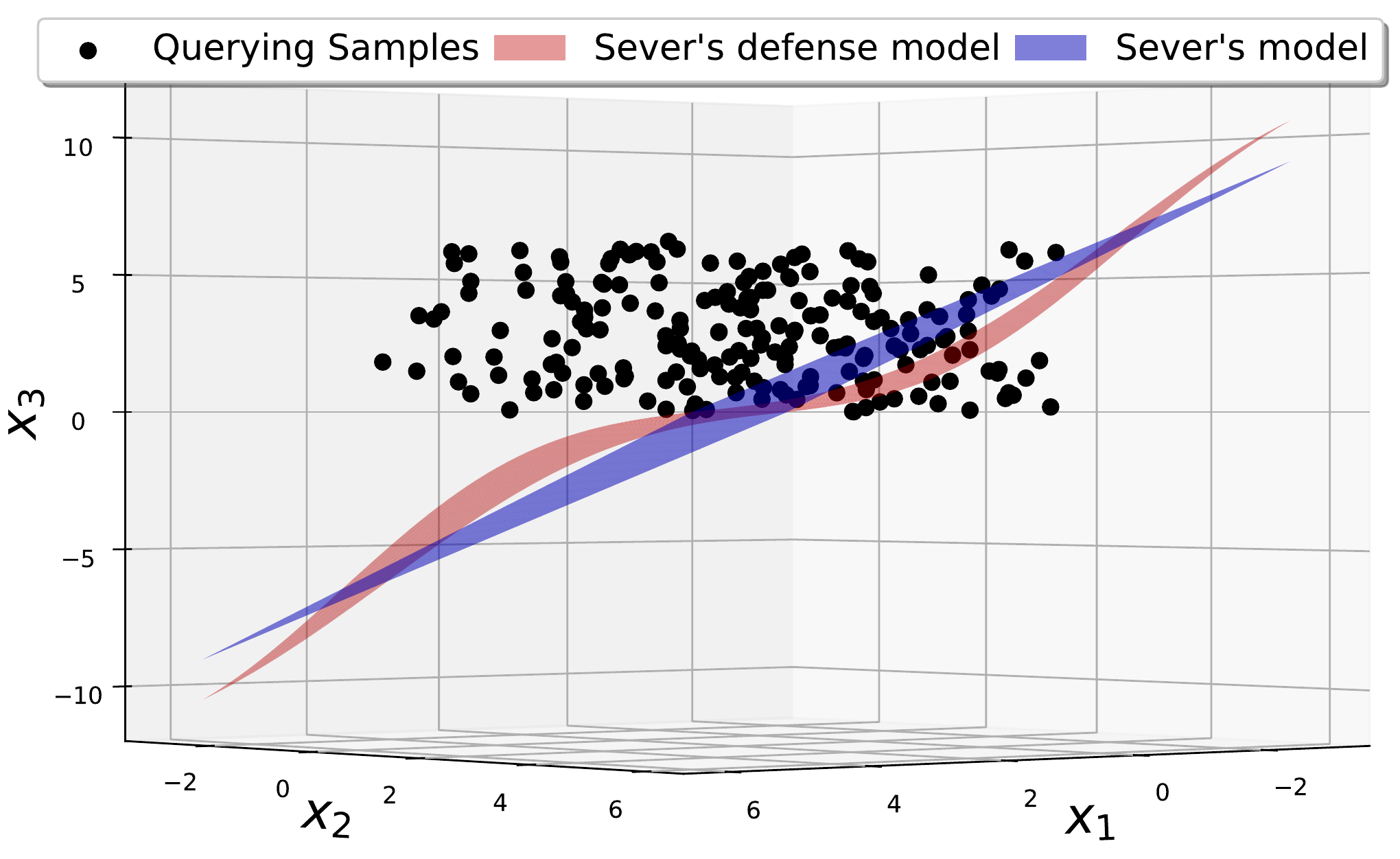}\label{fig_ill_cas2}}
	\vspace{-0.1in}
	\caption{
	Illustrations of the two pairs of attack-defense strategies: (a) server's utility: 86\%, adversary's utility: 92\%; (b) server's utility: 98\%, adversary's utility: 90\% (stand errors are within 0.01 over 50 replications).
	}\label{examples}
	\vspace{-0.3cm}
\end{figure*}
Suppose that the server chooses a defense strategy that severely distorts the returned values, namely the $\fs^g$ and $\fs$ are very different. In that case, it is conceivable that both the server's and adversary's utility will degrade. However, low-quality prediction services are undesirable for a benign user whose only focus is the in-sample prediction performances. 
As a result, the server inevitably faces a tradeoff between the utility it maintains through the APIs prediction service and the loss of its machine learning model's privacy (or the adversary's utility).


To better understand such tradeoffs, we consider the following illustrative examples.
Suppose that the server has a binary classification model $\fs$, which is a hyper-plane containing the origin in $\mathbb{R}^3$.
Suppose that the defense model is $f \circ g$, where $g:u \mapsto m\cdot\sin(\omega \cdot u)$. We consider two different pairs of attack-defense strategies. 

\textbf{Case 1}. The adversary sends the queries $\bbx$, which is uniformly sampled from the box $[-3,3]^3$ (Fig.~\ref{fig_ill_case1}). It trains a logistic regression model upon the responses from the server. The server uses the strategy with $m= 2 , \omega= 3$. It can then be numerically calculated that the server's utility is $0.86$, and the adversary's utility is $0.92$, based on the 0-1 loss.
The numerical calculation was based on $10^5$ test data, so that the $6\%$ gap is significant. 

\vspace{0.05cm}
\textbf{Case 2}. The adversary sends the queries $\bbx$, which is uniformly sampled from the box $[0,6]^3$ (Fig.~\ref{fig_ill_cas2}). The server uses the strategy with $m= 0.8, \omega=1 $. Other settings are the same as the first case. It can be calculated that the server's utility is $.98$, and the adversary's utility is $0.9$.

The difference between the two cases is the informativeness of both query samples (adversary's choice) and defense strategy (server's selection). Specifically, in the first case, the query samples across the decision boundary are balanced (in labels $\{1,-1\}$). Therefore, they bring much more information than the very unbalanced one in Case 2 for the adversary to extract the server's model. From the server's point of view, its defense model in Case 1 is not efficient since all those perturbed points are distributed symmetrically. On the other hand, the defense mechanism in the second scenario fallaciously increases the labels of the minority class and hence downgrades the adversary's utility more.

The server favors the second case since it degrades the adversary's utility by a greater margin compared to the loss in its utility. In contrast, the adversary favors the first scenario due to a relatively high adversarial utility. To systematically evaluate the quality of each pair of defense-attack strategies, we propose the following metric. 

\vspace{0.05cm}
\begin{definition}[Adversary-Benignity curve]
An Adversary-Benignity (AB) curve is a set of points $(b, a)$ associated with a set of server-adversary strategy pairs, where $b$ and $a$ are the utilities of the server and adversary evaluated on the given pair $(f_A, f_S)$, respectively. 
\end{definition}

Consider a two-dimensional plane where the $x$-axis is the server's utility, and the $y$-axis is the adversary's utility. Then by constructions, all the points inside the unit box with the lower-left corner at the origin are valid components for the AB Curve. For the two cases mentioned above, their corresponding AB curves are two points, namely $(0.86,0.92)$ and $(0.98,0.9)$. 
In general, one may naturally ask the following question: what is the shape of the AB curve of a set of strategies that the server/adversary favors? Intuitively, the server would like a pair of the strategy set with its utility that is located below the $45^{\circ}$ line, because such pair results in a higher server's utility than the adversary's utility. In other words, the server deliveries a better prediction performance with a lower privacy loss.

\vspace{0.3cm}
\section{Probabilistic View: A study of the AB Curve}\label{abc_analysis}

As mentioned above, the $45^{\circ}$ line of the AB curve serves as a standard to check the relative efficiency of a pair of attack and defense.
It is essential to study which kind of defense-attack pair will lead to a $45^{\circ}$ line of the AB Curve, as it corresponds to an equilibrium scenario where server's utility equals to adversary's utility. 
We consider the following setting to simplify the analysis. The server holds a binary classification model. Upon being queried, it will only return hard-threshold labels. The evaluation will be based on the $0-1$ loss function. In other words, $\mathcal{L}(y,\sgn(f(x))) = \ind\{y \neq \sgn(f(x))\}.$
Additionally, we assume that both the server and the adversary have the same set of the model classes, e.g., neural networks.

Let us consider the following scenario where the adversary attacks the server by querying i.i.d. samples from a distribution $Q$, and the server applies a probabilistic method to protect its models. In particular, we consider the following probabilistic defense strategy set. Let 

\begin{align}\label{stat2}
  \fs^g(x) =
    \begin{cases}
      \fs(x) & \text{if $x \not\in \mathcal{S}$}\\
       B \cdot \fs(x) & \text{otherwise},
    \end{cases} 
\end{align}
where $\mathcal{S}$ is a subset of the input space $\X$, $B$ is a symmetric Bernoulli random variable taking values from $\{+1,-1\}$, and denote $\ty = \sgn(\fs^g(x))$. 

To see that the above pair of attack-defense results on a $45^{\circ}$ line of the AB curve, we first show that $\fs$ is `not learnable' using the notion of learnability from the learning theory~\cite{valiant1984theory,schapire1990strength}. 
From the adversary's view, the process of querying the server and building a model to mimic the server's behavior can be regarded as a `learning' process in learning theory. 
In our context, the adversary will only access the server's \textit{defense oracle} $\fs^g$, which is the defense version of $\fs$. We define the following, which slightly modifies the original notion of learnability.

\begin{definition}[$\v$-learnable via Algorithm $\mathcal{A}$]\label{def_learn}
A function $f \in \mathcal{F}$ is $\v$-learnable ($\v \geq 0 $) via the algorithm $\mathcal{A}$ by a hypothesis class $\mathcal{H}$ if for all $\mu \in (0,1/2),$ $ \delta \in (0,1/2)$, and a distribution $D$ over $X$, given access to the defense oracle $f^g$, the algorithm $\mathcal{A}$ runs in polynomial time in $1/\mu$, $1/\delta$, and dimension of $\mathcal{F}$ to output a $h \in \mathcal{H}$, such that $P_D(f(x) \neq h(x)) < \v + \mu$ with probability at least $1-\delta$.
\end{definition}

In the above definition, a smaller $\v$ corresponds to a more accurate learning of $f$ by $h \in \mathcal{H}$. With $\v = 0$, it reduces to the standard notion of learnability. Regarding the learning algorithm $\mathcal{A}$, it is standard to use the risk minimization principle, namely to solve $\operatorname{min}_{h} \mathcal{L}(f^s_g(x), h(x))$. The risk minimization will be set as the default learning algorithm in the subsequent analysis.

We show that the server's model $\fs$ is not learnable with the pair of attack-defense defined in Eq.~(\ref{stat2}) under the following assumption. Recall that $\mathcal{F}_S$ is the server's model class. 

\begin{assumption}\label{assump1}
For the querying distribution $Q$ over $\X$, the server can find a subset $\mathcal{S} \subset \X$ with $P(x \in \mathcal{S})= 2\v$ for some $\v \in (0,0.5)$, and a $f_B \in \mathcal{F}_S$, such that $\sgn(f_B(x)) \neq \sgn(\fs(x))$ for $x \in \mathcal{S}$, and $\sgn(f_B(x)) = \sgn(\fs(x))$ for $x \not\in \mathcal{S}$. 
\end{assumption}

Intuitively speaking, the assumption requires the existence of a model $f_B \in \mathcal{F}_S$ that behaves similarly but not precisely to the server's model $\fs$.

\begin{theorem}\label{theo1}
Under Assumption \ref{assump1}, for any $ \gamma \in (0, \v)$,  server's classifier $f_S$ is NOT $\gamma$-learnable by
the adversary's model class $\mathcal{F}_A$ via risk minimization algorithm.
\end{theorem}

Recall the server's utility for its true model $\fs$ is 1 by definition. The above result suggests that if the server's defense model $\fs^g$ delivers a server's utility of $1-\v$, then the adversary cannot learn a $f_A \in \mathcal{F}_A$ that is more accurate than $\v$-close to the server's true model $\fs$ i.e., $P_Q(\fs \neq \fa) = \v$  with Risk Minimization algorithm. 

Now we derive the $45^{\circ}$ line of AB curve from the result of Theorem~1. We assume that the expectation of adversary's utility (as defined in Def.~\ref{def_adv}) is taken with respect to the (sampling) distribution $Q$. Specifically, from the above theorem, we can find (details in Section~\ref{case_study} in the appendix) that the adversary faces a dilemma where there are two indistinguishable functions that can server as the risk minimizers i.e., there exists a $f_B \in \mathcal{F}_S$ such that $\mathbb{E}_{x,\tilde{y}}\loss(\tilde{y}, \sgn(\fs(x))) = \mathbb{E}_{x,\tilde{y}}\loss(\tilde{y}, \sgn(f_B(x)))$. 
As a result, if the adversary chooses to learn $\fs$ with risk minimization algorithms, then without additional information, choosing between $f_B$ and $f_S$ with a even coin flipping is the only choice for the adversary. Under such circumstances, the adversary's utility is calculated to be $1- \v$ (details in Appendix~A3), which equals to server's utility $1-\v$. Therefore, this kind of equilibrium relationship gives rise to a $45^{\circ}$ line of the AB Curve as illustrated by the solid black line in Fig.~\ref{abcurve}.

Having established the $45^{\circ}$ line AB curve of a pair of attack-defense, one may wonder what is a real-world example of such a pair of attack-defense.
One potential candidate for the defense mechanism is the boundary differentially private layer (BDPL)\footnote{As suggested by its name, BDPL is closely related the concept of differential privacy~\cite{dwork2004privacy}. The motivation in the original paper was to apply differential-private technique to protect the decision boundary.}~\cite{zheng2019bdpl}. Because in BDPL, the server will first select a neighborhood threshold $\delta$ of its classifier.  Upon receiving a querying point $x$, if $x$ lies in the $\delta$-neighbor of $\fs$, i.e., $d(x,\fs) \leq \delta$ for a distance measure $d(\cdot,\cdot)$, then the server will flip the label with probability $
1/2-\sqrt{e^{2 t}-1}/(2+2 e^{t})
$ for a privacy parameter $t \geq 0$. 
It can be verified that the BDPL is an instance of the defense mechanism discussed in Eq.~\ref{stat2} for setting $\mathcal{S}$ as the $\delta$-neighborhood of $\fs$'s decision boundary and $P(B=-1) = 1/2$.



\begin{figure}
    \centering
    \includegraphics[width=6.5cm, height=6.9cm]{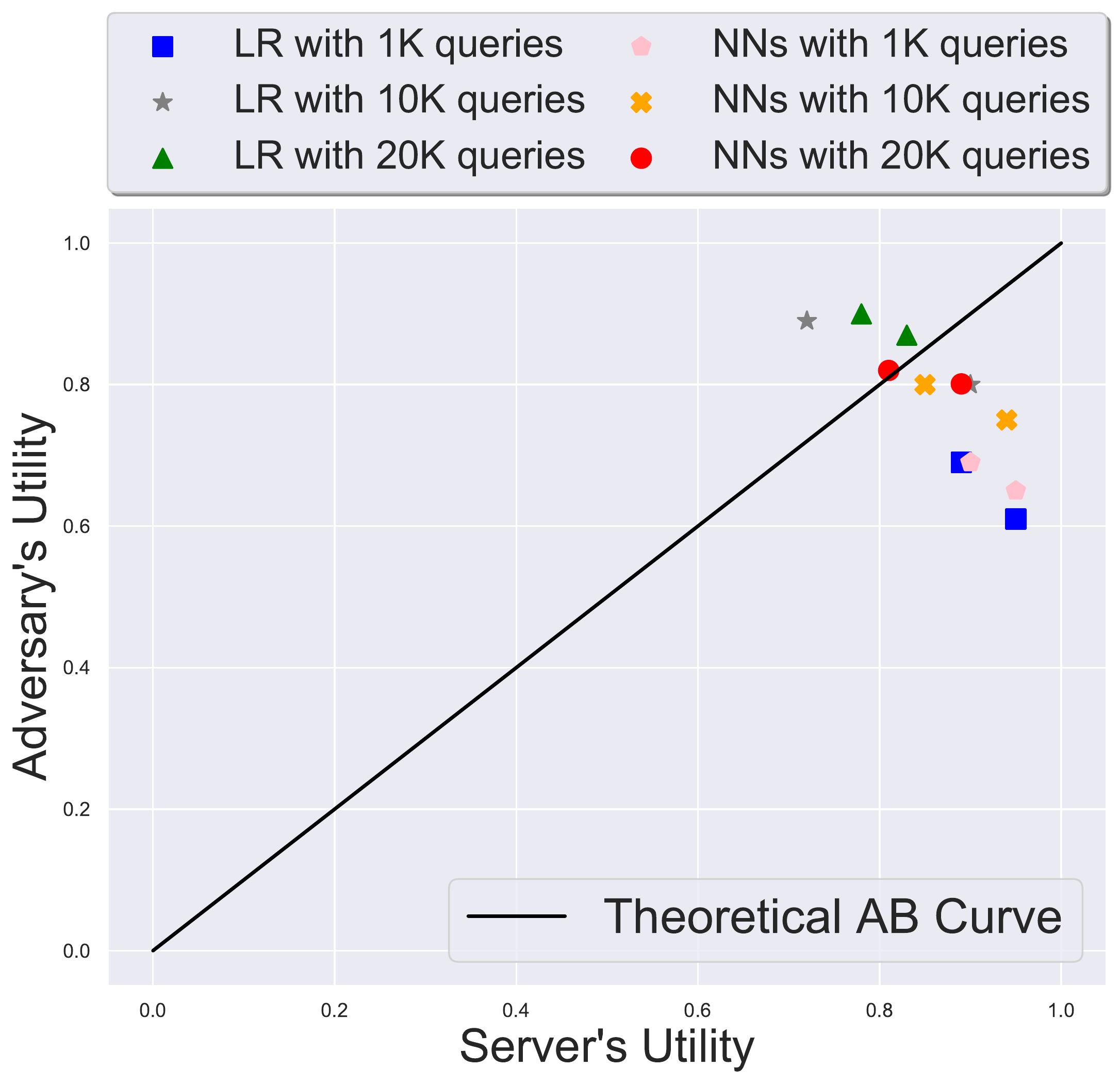}
    \vspace{-0.1in}
    \caption{Illustrations of both the theoretical AB curve and the empirical results reported in~\cite{zheng2019bdpl}. Each shape of utility pair has two entries since they are evaluated on \textit{Mushroom} and \textit{Adult} datasets, respectively. }\label{abcurve}
    \vspace{-0.2in}
\end{figure}

How are the empirical results of BDPL related to our theoretical derivations of a $45^{\circ}$ line of AB curve of an attack-defense pair?
The experimental goal in~\cite{zheng2019bdpl} is to study the effectiveness and efficiency of BDPL under Model Extraction Attack, and their setup is as follows. They trained a logistic regression (LR) model and neural networks (NNs) as the server's models on the \textit{Mushroom} and the \textit{Adult} dataset (both are from UCI machine learning repository~\cite{Dua:2019}). For the BDPL parameters, they set the boundary threshold $\delta = 0.125$ and the DP parameter to be $\v = 0.01$. Therefore the flipping probability inside the decision boundary area is $0.48$.
In Fig.~\ref{abcurve}, we plot several pairs of server-adversary's utility for two models, namely LR and NNs, with different numbers of adversary's querying sample size reported in their paper. (The definitions of both the server's and the adversary's utility are the same.) Given 1K of adversary's querying samples, the server-adversary utility pairs sit in the region with around $90\%$ of server's utility and $62\%$ of adversary's utility. This is reasonable since the adversary does not have access to enough data at the beginning. As the querying number increases (1K $\rightarrow$ 10K $\rightarrow$ 20K), the server-adversary utility pair moves towards the point $(0.8,0.8)$. 
Hence, we conclude that the empirical results match the theoretical derivation quite well as the number of querying sample sizes increase.

\section{Optimization View}\label{sec_opt_algo}

\textbf{Min-max formulation}
Built upon the `equilibrium' case exemplified previously, we now discuss how to find a more efficient strategy (compared to the $45^{\circ}$ AB Curve case).
From the model extraction and defense process described in the ``Model Extraction and Defense'' section, 
we can formulate the process as follows.
\begin{align}\label{gene_for}
    \min^{(A)}_{\bbx \in \R^{np}}  \max^{(S)}_{g \in \mathcal{G} } \min^{(A)}_{\fa \in {\mathcal{F}}_A(\bbx, \widetilde{\bby}^g_n)} \mathbb{E}_{(\testx,\testy) \sim Q} \ \mathcal{L}(\testy,\fa(\testx)), 
\end{align}
where $\bbx$ denotes the adversary's queried samples, $n$ is the query size, $ \mathcal{G}$ is the server's set of defense strategies, ${\mathcal{F}}_A$ is the adversary model class,  $\widetilde{\bby}^g_n = \fs^g(X)$ represents the server's perturbed response, $\fa \in {\mathcal{F}}_A(\bbx, \widetilde{\bby}^g_n)$ means that the adversary will select a model $\fa \in \mathcal{F}_A$ based on the information $(X,\widetilde{\bby}^g_n)$, and finally the expectation is taken with respect to future test data $(\testx, \testy) \sim Q$. 
Note that the loss function may depend on specific tasks. For example, the loss function can be scaled $L_2$ distance, i.e., $\loss(y,f(x)) = |y-f(x)|^2/|y|^2$ in regression, or indicator function $\loss(y,f(x)) = \mathbf{1}(y(x) \neq f(x))$ in classification. 
	

Next, we provide more detailed interpretations of the above formulation.
We first explain each layer of the min-max. For a given attack strategy (as described by the set of queries $\bbx$), a defense strategy (as indexed by $g$), the problem is to solve the sub-problem (for (A))
$\min_{\fa \in {\mathcal{F}}_A(\bbx, \widetilde{\bby}^g)} \mathbb{E} \ \mathcal{L}(\testy,\fa(\testx)).$
In this sub-problem, the adversary solves an empirical risk minimization problem from labeled data $\bbx, \widetilde{\bby}^g$, and the loss of $\fa$ will be evaluated at the expectation over test data distribution.
The second layer $\max^{(S)}_{g \in \mathcal{G}}$ is to select $g$ from a set of defense strategies.
The outer layer $ \min^{(A)}_{\bbx \in \R^{np}}$ represents the adversary's strategy to deliberately craft a certain number of queries to maximize the adversary's utility. We should emphasize that the above formulation is from the adversary's perspective, i.e., designing optimal samples in the awareness of the server's potential defense since the outermost problem corresponds to the min ($X$) problem.

\textbf{Algorithms}
In this section, we design practical algorithms to characterize how the adversary and the server can (approximately) solve problem \eqref{gene_for}.
As this problem is extremely challenging, we will make a few simplifications. 

First, under the typical black-box model assumption,
the server will only return discrete labels i.e.,
$\widetilde{\bby}_n =  [\sgn(\fs^g(x_1)),$
$ \sgn(\fs^g(x_1)),\sgn(\fs^g(x_2)), ...,\sgn(\fs^g(x_n))],$ corresponding to $\bbx_n =[x_1, x_2, ..., x_n]$.
In order to develop implementable algorithm, we need to relax the assumption in the sense that the server now returns soft labels i.e., $\widetilde{\bby} = [\sigma(\fs^g(x_1)), \sigma(\fs^g(x_2)),...,\sigma(\fs^g(x_n)) ] $. For notation simplicity, we would simply drop the subscript $n$.
Second, we make the following simplification to restrict the server's abilities:  We discretize the defense strategy functional space $\mathcal{G}$, and only allow the server to select from a given set of strategies $g_1, g_2, ...,g_m \in \mathcal{G}$. This way, the server only optimizes the selection vector, but not the strategy itself. 
A possible solution for the server to obtain the best defense strategy is to optimize the selection over a probability simplex $\Delta = \{\lambda_{1:m}| \sum_{j=1}^m \lambda_j = 1, \lambda_j \geq 0 \text{ for } j =1,2,...,m\}$. One benefit for adopting such discretization is that it simplifies theoretical convergence arguments.
Third, we assume that the adversary's model is fully specified (parameterized) by $\beta \in \mathbb{R}^d$,  denote as  $\fa(\cdot;\beta) \in \mathcal{F}_A$ (adversary's model class) to represent the relationship. Moreover, the adversary is allowed to choose query samples across the entire space $\mathcal{X}$. 
Since there are $m$ defensive strategies, then the adversary will train $m$ different models correspondingly. We denote $\bm{\beta} = [\bm{\beta}_1, \bm{\beta}_2, ...,\bm{\beta}_m ] \in \mathbb{R}^{m \times d}$.

Under the above setting, and given a loss function $\mathcal{L}(\cdot,\cdot)$, we obtain a new min-max bi-level optimization problem:
\begin{align}
\begin{split}
\min_{x \in \mathcal{X}}  \max_{\bm{\lambda} \in \Delta}\, & H(\bm{\beta}^*(\bbx),\bm{\lambda})= \frac{1}{m}\sum_{j=1}^m H_j(\bm{\beta}^*_{j}(\bbx))\lambda_j \\
  &=\frac{1}{m}\sum_{j=1}^{m} \frac{1}{N}\sum_{t=1}^{N} \loss(\underline{\testy}_t,\sigma(\fa(\underline{\testx}_{t}; \bm{\beta}^*_{j} )))\lambda_j,
  \end{split}
\label{bilevel} 
\\[2ex]
\text{subject to}\quad & \bm{\beta}^*_{j} = \argmin_{\bm{\beta}_j \in \mathbb{R}^d}h_j(\bbx,\bm{\beta}_j) \text{ for } j = 1,2,...,m,\nonumber
\end{align}
where 
$$ h_j(\bbx,\bm{\beta}_j) =  \sum_{i=1}^n \loss(\sigma(\fa(x_i; \bm{\beta}_j)), {\sigma(\fs^{g_j}(x_i))}) + \mathcal{R}(\bm{\beta}_j)$$ with $\mathcal{R}$ being a regularizer.
Intuitively, the upper-level problem $H(\cdot,\cdot)$  evaluates the quality of adversary's model $\fa$ 
over future testing data $\{(\testx_t, \testy_t)\}_{t=1}^N$, and the lower-level problem i.e., $h_j$ for $j=1,2,...,m$ corresponds to adversary's model building process. 
Such a problem is quite challenging to solve since it couples {\it three} sub-problems, the lower-level loss minimization problem, the upper-level maximization problem (where the server generates defense strategies), and the upper-level minimization problem (where the adversary generates samples). 

\begin{algorithm}[tb]
		\centering
		\caption{Min-Max Bi-level Approximation Stochastic Gradient Descent-Ascent}\label{algo1}
		\footnotesize
		\begin{algorithmic}[1]
			\renewcommand{\algorithmicrequire}{\textbf{Input:}}
			\renewcommand{\algorithmicensure}{\textbf{Output:}}
			\REQUIRE 			\textit{Initialization}: 
			$\lambda_j^0 =1/m$, $\bm{\beta}_j^0 \sim \mathcal{N}(\mathbf{0},I_d)$ for $j=1,...,m$, $\bbx^0 \sim \mathcal{N}(0,I_{n \times d})$, stepsizes $\{ r_k, s_k\}_{k \geq 0}$, and mini-batch $b$.
			 \\\hrulefill
            \FOR{$k=0,1,2,...,K-1$}
            \FOR{$j=1,2,...,m$}
            \REPEAT
            \STATE Solve $\bm{\beta}^*_{j} = \argmin_{\bm{\beta}_j \in \mathbb{R}^d}h_j(\bbx,\bm{\beta}_j)$ with gradient descent on $\beta_j$.
            \UNTIL
            \STATE Stop criteria satisfied$^{\ddagger}$.
            \ENDFOR
            \STATE $\bbx^{k+1} = \text{Proj}_{\mathcal{X}} (\bbx^{k} - r_k \overline{\nabla}_{\bbx} H^{k\dagger}) $
            \STATE $\bm{\lambda}^{k+1} = \text{Proj}_{\Delta}(\bm{\lambda}^{k} + s_k \overline{\nabla}_{\bm{\lambda}} H^{k\dagger}) $
			\ENDFOR
	        \STATE Randomly draw $\overline{\bbx}$ from $\{\bbx^k\}_{k=1}^{K}$ with uniform probability.
            \\\hrulefill\\
   $\dagger$ At $k$th iteration,       $\overline{\nabla}_{\bbx} H^k$, and $\overline{\nabla}_{\bm{\lambda}}H^k$ are stochastic estimates of $\nabla_{\bbx}H(\bm{\beta}^{k+1}(\bbx^k), \bm{\lambda}^k)$, and $\nabla_{\bm{\lambda}}H(\bm{\beta}^{k+1}(\bbx^{k+1}), \bm{\lambda}^k)$ respectively.
   \\ $^*$ 	$\text{Proj}_{\mathcal{X}}$, $\text{Proj}_{\Delta}$ are the Euclidean projection operators onto set $\mathcal{X}$ and $\Delta$ respectively.
   \\ $^{\ddagger}$ For a given $\epsilon > 0$, the procedure is terminated if $\| \nabla_{\beta_j}h_j(\bbx, \beta_j)\| \leq \epsilon$.
			
			\ENSURE $\overline{\bbx}$ and  ${\bm{\lambda}}^{K}$ .

		\end{algorithmic}
	\end{algorithm}

Motivated by these recent works, we develop a \textit{Min-Max Bi-level Stochastic Approximation Gradient Descent-Ascent} algorithm, as summarized in Algo.~\ref{algo1}. It consists of three steps of (stochastic) gradient descent. The algorithm first operates with the inner (lower level) problem in Line 2-4, in which it runs gradient descent on each of the $m$ sub-problems. Then it switches to the upper-level problem and runs gradient descent on the minimization ($\bbx$) problem and gradient ascent on the maximization ($\bm{\lambda}$) problem in Line 5-6. 

One key design consideration is that the three different kinds of updates must be executed at different ``speeds" (i.e., time scales). 
This is because, for example,  $\bm{\lambda}^{k}$ may not stay close to $\bm{\lambda}^*(\bbx^{k})$ ($\bm{\lambda}^*$ is the optimal solution of Eq.~(\ref{bilevel})) at each iteration $k$, and therefore it is not guaranteed $\nabla_{\bbx}H(\bm{\beta}^{k+1}(\bbx^k), \bm{\lambda}^k)$ would deliver a `true' descent of the objective. A similar situation arises for the lower-level problem as well. Therefore, in the proposed algorithm, the lower-level variable $\beta$ is updated in the fasted time-scale (i.e., solve to exact global min), the maximization problem is solved relatively slowly (by using some appropriate stepsizes $c_k$), while the $X$ variable is updated most slowly (by using some very small stepsizes $b_k$). 
Besides the above theoretical arguments, it is piratically meaningful to adopt such rules for optimization. First, the upper-level problem, i.e., evaluation of adversary's model, is preceded by constructing the model itself. Therefore, it is favorable to make the lower-level problem converge first. 

\textit{Convergence result.} 
We state the convergence result of the proposed Algo.~\ref{algo1}. We need the following 3 assumptions.

\begin{assumption}\label{assump2}
For any function $H(\bm{\beta}, \bm{\lambda})$ and  $e(\bbx) \coloneqq H(\bm{\beta}^*(\bbx),\bm{\lambda}^*)$.
\begin{itemize}
\setlength\itemsep{0em}
   
    \item For any $\bbx \in \mathcal{X}$, $\nabla_{\bbx}H(\cdot,\bm{\lambda})$ and $\nabla_{\bm{\beta}}H(\cdot,\bm{\lambda})$ are Lipschitz continuous with respect to (w.r.t.) $\bm{\beta}$.
    \item  For any $\bbx \in \mathcal{X}$ and $\bm{\beta} \in \mathbb{R}^d,$ we have $\| \nabla_{\bm{\beta}}H(\bm{\beta},\bm{\lambda})  \| \leq C_{H},$ 
for some $ C_{H}> 0.$
    \item The function $e(\cdot)$ is weakly convex with modulus $\mu_{\ell} \geq 0$: $$ e(w) \geq e(v)+\langle\nabla e(v), w-v\rangle+\mu_{\ell}\|w-v\|^{2}, \forall w, v \in X. $$
    \item $H(\bm{\beta}, \bm{\lambda})$ is $l$-smooth and $H(\cdot,\bm{\lambda})$ is $L$-Lipschitz for each $\bm{\lambda} \in \Delta$ and $H(X,\cdot)$ is concave for each $\bbx \in \mathcal{X}.$
   
\end{itemize}
\end{assumption}

\begin{assumption}\label{assump3}
Assumptions on the lower-level functions $h_j(\bbx,\bm{\beta}_j)$ for $j=1,2,...,m.$
\begin{itemize}
\setlength\itemsep{0em}
    \item For any $\bbx \in \mathcal{X}$ and $\bm{\beta}_j \in \mathbb{R}^d,$ $h_j(\bbx,\bm{\beta}_j)$ is twice continuously differentiable in $(\bbx,\bm{\beta}_j)$;
    \item For any $\bbx \in \mathcal{X}$,  $\nabla_{\bm{\beta}_j}h(\bbx,\cdot)$ is Lipschitz continuous w.r.t. $\bm{\beta}_j \in \mathbb{R}^d,$
 and with constants $L_{h_j}$.
    \item For any $\bbx \in \mathcal{X}$, $h_j(\bbx, \cdot)$ is strongly convex in $\beta_j$, and with modulus $\mu_j \geq 0$.
    \item For any $\bbx \in \mathcal{X}$, $\nabla^2_{\bbx\bm{\beta}_j}h_j(\bbx,\cdot)$ and $\nabla^2_{\bm{\beta}_j\bm{\beta}_j}h_j(\bbx,\cdot)$ are Lipschitz continuous w.r.t. $\bm{\beta}_j \in \mathbb{R}^d.$
    \item For any $\bbx \in \mathcal{X}$ and $\bm{\beta}_j \in \mathbb{R}^d,$ we have $\|\nabla^2_{\bbx\bm{\beta}_j}h_j(\bbx,\bm{\beta}_j) \| \leq C_{h\bbx j}$ for some $C_{h\bbx j} > 0$.
     \item For any $\bbx \in \mathcal{X},$ $\min_{\bm{\beta_j}} h_j(X,\beta_j)$ has closed-form solution.

    \item For any $\bm{\beta}_j \in \mathbb{R}^d,$ $\nabla^2_{\bbx\bm{\beta}_j}h_j(\cdot,\bm{\beta}_j)$ and $\nabla^2_{\bm{\beta}_j\bm{\beta}_j}h_j(\cdot,\bm{\beta}_j)$ are Lipschitz continuous w.r.t. $\bbx \in \mathcal{X}$.
    
\end{itemize}
\end{assumption}

\begin{assumption}\label{assump4}
Assumptions on stochastic estimates for $k=1,2,...,K$.
\begin{itemize}
\setlength\itemsep{0em}
    \item $\mathbb{E}(\overline{\nabla}_{\bbx} H^k) = \mathbb{E}(\nabla_{\bbx}H(\bm{\beta}^{k+1}(\bbx^k), \bm{\lambda}^k)) $
, and $\mathbb{E}(\overline{\nabla}_{\bm{\lambda}}H^k) = \mathbb{E}(\nabla_{\bm{\lambda}}H(\bm{\beta}^{k+1}(\bbx^{k+1}), \bm{\lambda}^k))$ 
    \item $\mathbb{E}\|\overline{\nabla}_{\bbx} H^k - \nabla_{\bbx}H(\bm{\beta}^{k+1}(\bbx^k), \bm{\lambda}^k) \| \leq c^2$ for some $c>0$.
\end{itemize}
\end{assumption}

In addition to the above assumptions, we need the follow notations and definitions.
Denote $\Phi(\cdot)=\max _{\bm{\beta} \in \Delta} H(\cdot, \bm{\beta})$.

\begin{definition}
A function $\Phi_{\lambda}: \mathbb{R}^{m} \rightarrow \mathbb{R}$ is the Moreau envelope of $\Phi$ with a positive parameter $\lambda>0$ if $\Phi_{\lambda}(\mathbf{x})=\min _{\mathbf{w}} \Phi(\mathbf{w})+(1 / 2 \lambda)\|\mathbf{w}-\mathbf{x}\|^{2}$ for each $\mathbf{x} \in \mathbb{R}^{m}$
\end{definition}

\begin{definition}
 A point $\mathbf{x}$ is an $\epsilon$-stationary point $(\epsilon \geq 0)$ of a $\ell$-weakly convex function $\Phi$ if $\left\|\nabla \Phi_{1 / 2 \ell}(\mathbf{x})\right\| \leq$ $\epsilon$. If $\epsilon=0$, then $\mathbf{x}$ is a stationary point.
\end{definition}

Denote $\widehat{\Delta}_{\Phi}=\Phi_{1 / 2 \ell}\left(X_{0}\right)-\min _{X} \Phi_{1 / 2 \ell}(X)$ and $\widehat{\Delta}_{0}=\Phi\left(X_{0}\right)-H\left(X_{0}, \bm{\beta}_{0}\right).$

\begin{theorem}
Under Assumption~\ref{assump2},
~\ref{assump3} and~\ref{assump4}, with stepsizes chosen as $b_k=\Theta\left(\epsilon^{4} /\left(\ell^{3} (L^{2} +c^2) \right)\right)$ and $c_k = \Theta(\epsilon^2/lc^2)$, and with the mini-batch size of $1$, the iteration complexity of Algo.~\ref{algo1} to return an
$\epsilon$-stationary point is bounded $$O((\frac{\ell^{3} (L^{2}+c^2)  \widehat{\Delta}_{\Phi}}{\epsilon^{6}}+\frac{\ell^{3}  \widehat{\Delta}_{0}}{\epsilon^{4}})\max\{1,\frac{c^2}{\epsilon^2}\}).$$
\end{theorem}

\begin{proof}
 Note that in Line 4 of Algo.\ref{algo1}, with Assumption~\ref{assump3}, each lower-level problem (i.e., $\beta_j$ for $j=1,2,\ldots,m$) is solved exactly with closed-form solution. Hence,  Problem~(\ref{bilevel}) simply reduces to a min-max optimization problem, whose convergence results follow from~\cite{lin2020gradient} under Assumption~\ref{assump2},and~\ref{assump4}. 
\end{proof}

We remark that the result only guarantees that the algorithm will visit an $\epsilon$-stationary point within a
certain number of iterations and return $\bar{X}$ which is drawn from $\{X_t\}_{t=1}^K$ at uniform. This does not guarantee
that the last iterate $X^K$ is the desired $\epsilon$-stationary point. Such a scheme and convergence result are standard in
nonconvex optimization for SGD to find stationary points. In practical implementation, we stop the algorithm as long as the training loss no longer decrease significantly.

\textit{Run-time Analysis.} We provide a run-time analysis for the Algo.~\ref{algo1} in the following. The key message is that 
the computation will be relatively efficient if the mini-batch size $b$ and the model parameter size $d$ are moderately large.

The main computational burden comes from the inverse operation on the hessian matrix with respect to the model parameters ${\beta}$ (with size $d$) for the update of $X$ in Line 7 of Algo.~\ref{algo1}. To alleviate this, we use Neumann series, namely $A^{-1} = \lim_{i \rightarrow \infty}\sum_i A^i$ to approximate the inverse of hessian with computational complexity $O(ed^2)$, where $e$ is the approximation steps used. 
For the update of the model parameter ${\beta}$ in Line 4, we use SGD with mini-batch size $b$ and the computational cost for each iteration is $O(b+d)$.
For the update of the weight parameter ${\lambda}$ (with size $m$) in Line 9, the computational cost for each iteration is $O(b+m)$.
Hence, the computational cost of the entire algorithm is $O(K(b+d)mq + Ked^2 + K(m+b))$, where $K$ is the total iterations of the algorithm and $q$ is total iterations of the step for Line 4. The computation will be relatively efficient if the mini-batch size $b$ and the model parameter size $m$ are moderately large.

\vspace{-0.15cm}
\section{Experimental Study}
\begin{figure*}[!htb]
	\centering
	\subfloat[][]{\includegraphics[width=4.2cm,height =4.2cm]{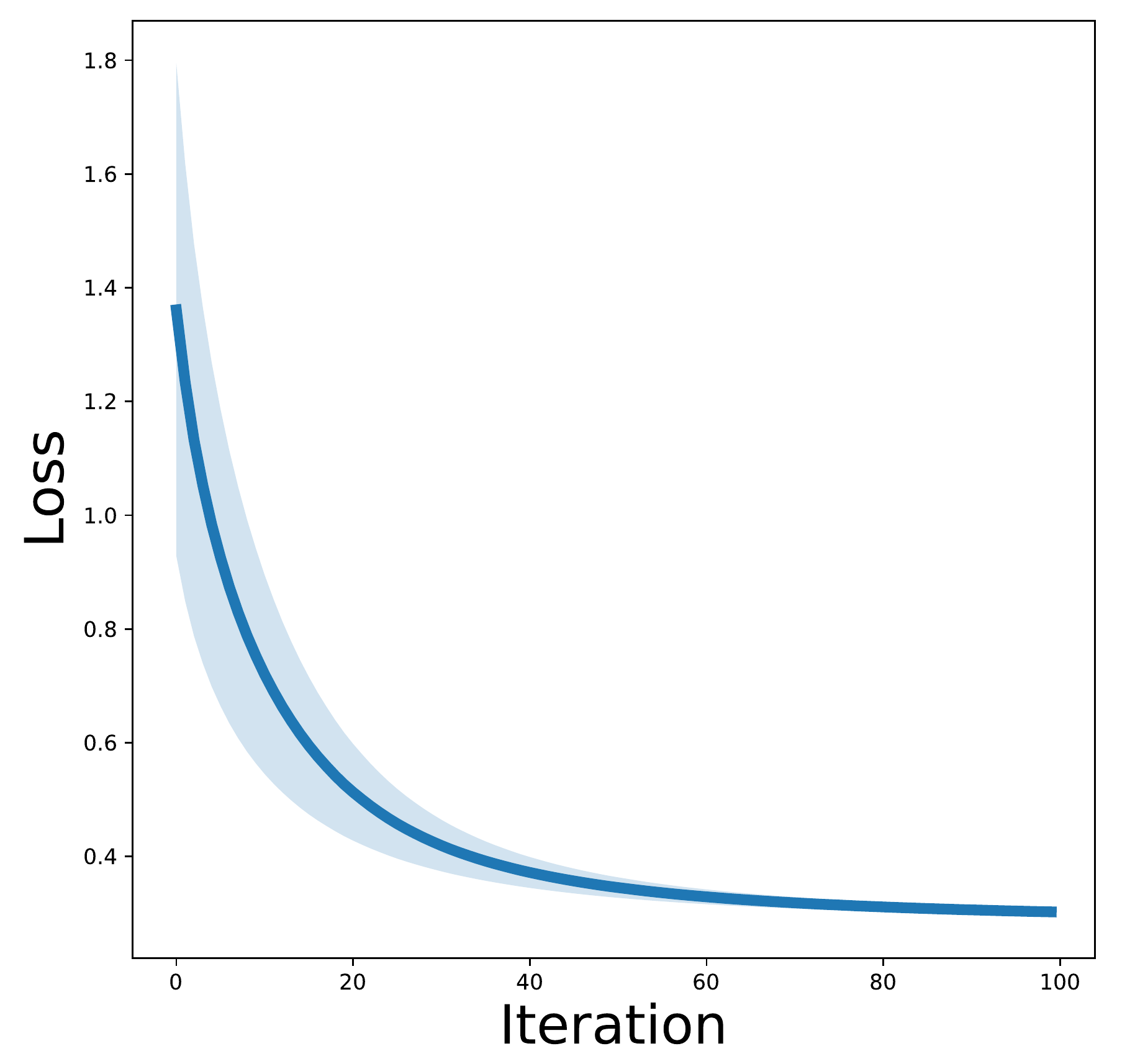}\label{train_loss}}\hfil
	\subfloat[][]{\includegraphics[width=4.2cm,height =4.2cm]{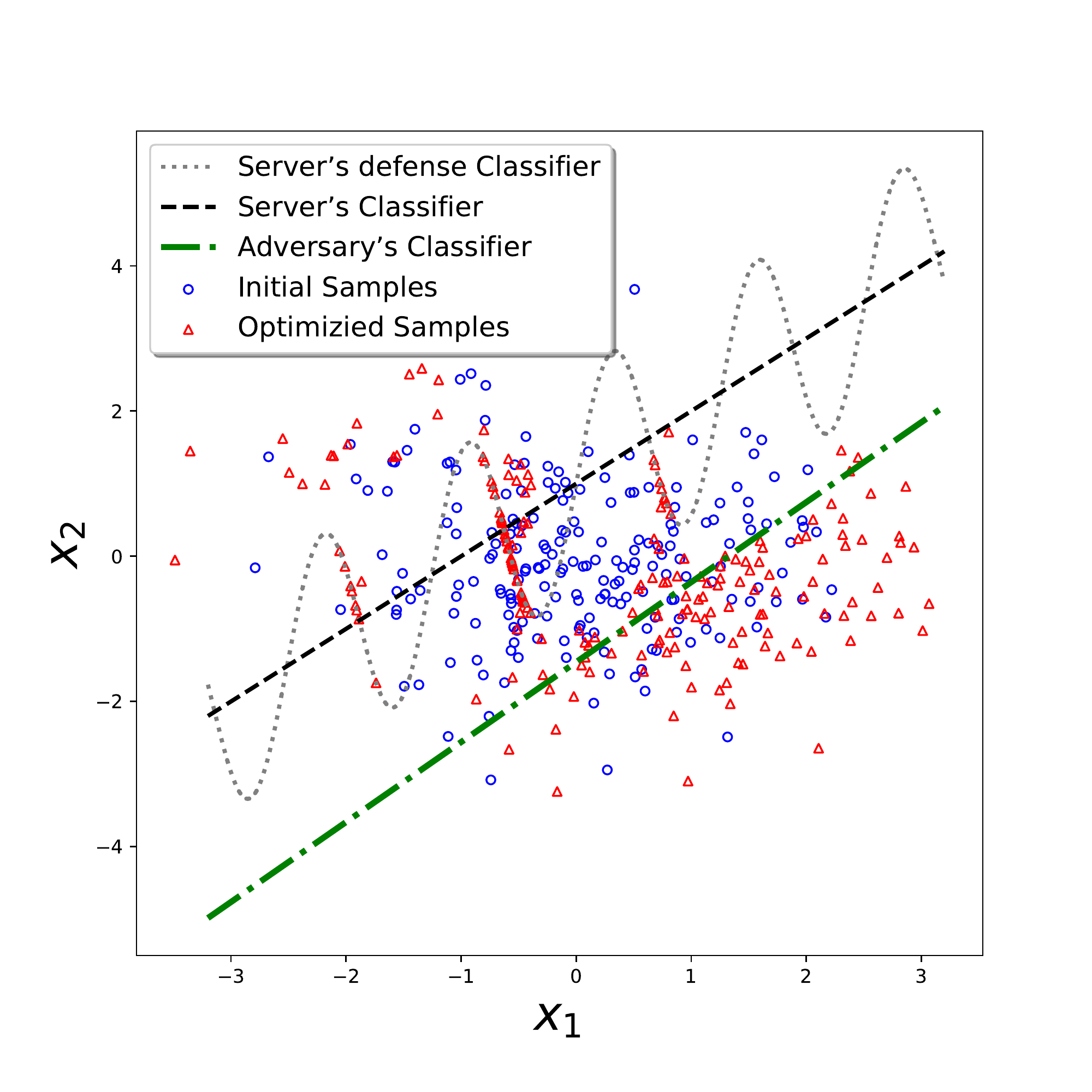}\label{90th}}\hfil
		\subfloat[][]{\includegraphics[width=4.2cm,height =4.2cm]{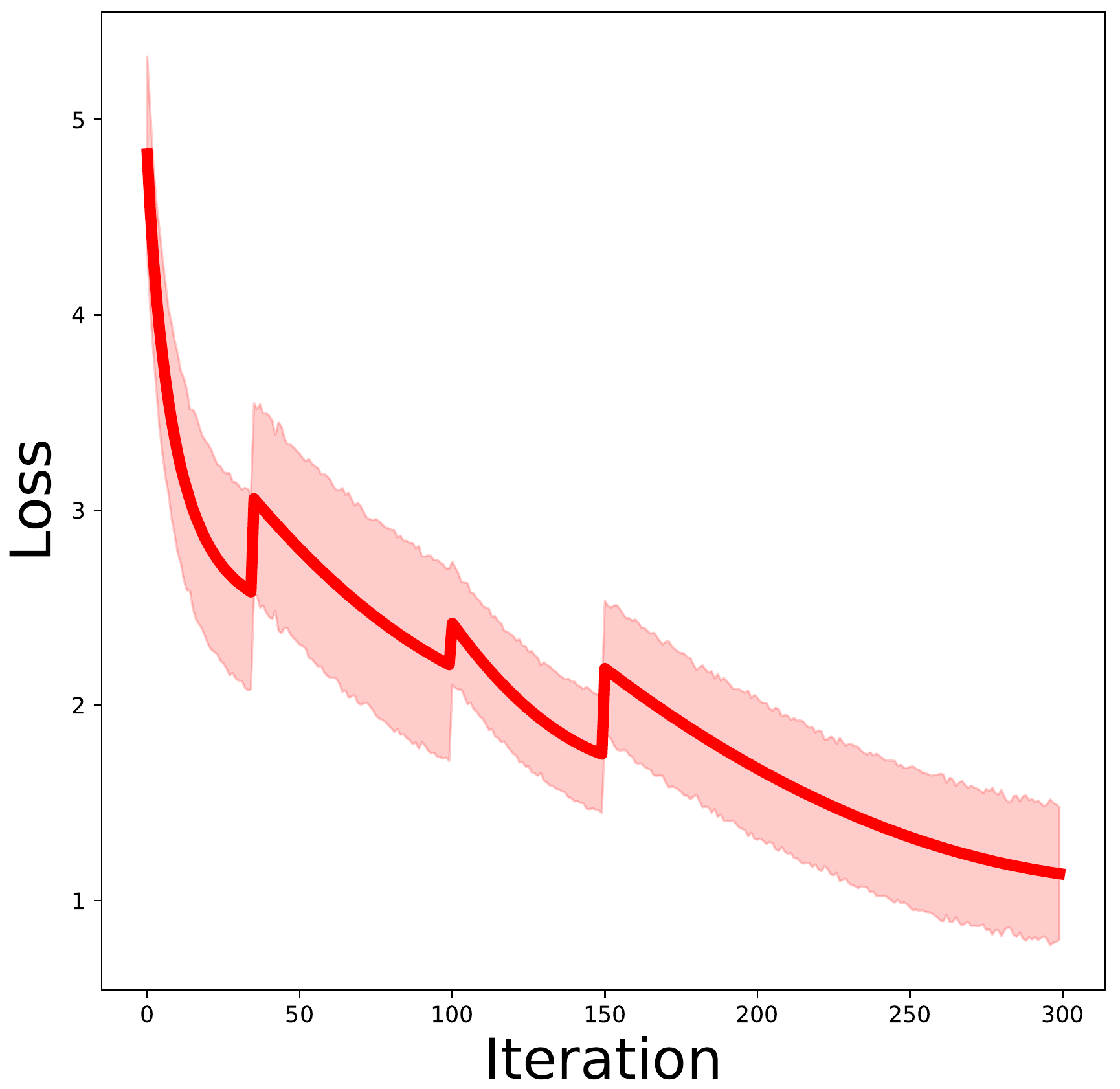}\label{tsne}}\hfil
	\subfloat[][]{\includegraphics[width=4.2cm,height =4.2cm]{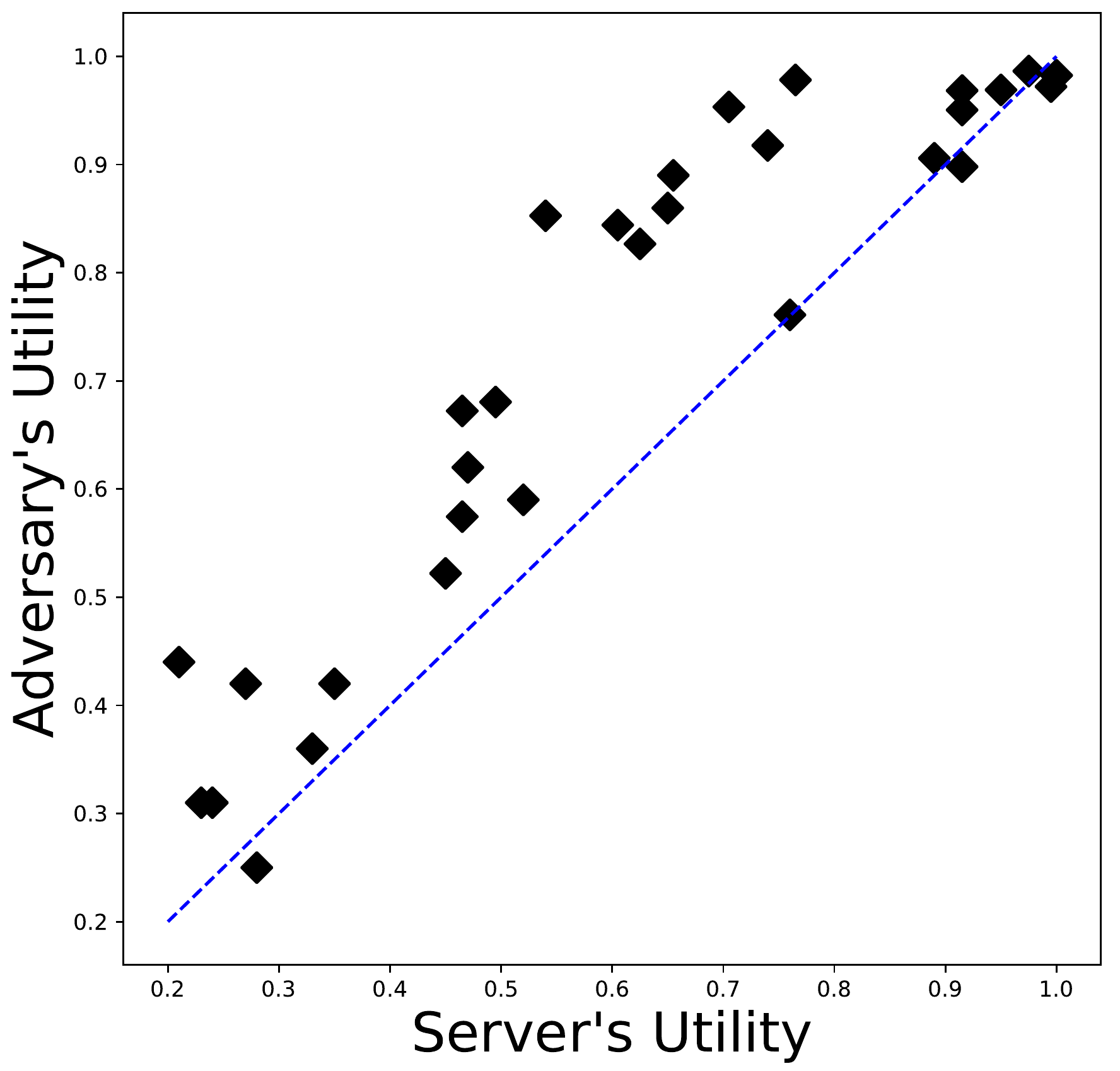}\label{abc}}
	
	\vspace{-0.1in}
	\caption{Results of the proposed Algo.~\ref{algo1} on (a) \& (b) synthetic linear data and (c) \& (d) MNIST dataset.
    Fig.~(a) illustrates the training loss on synthetic data over 20 replications where the solid blue line is the mean, and the shaded region represents the standard deviation. Fig.~(b) is a snapshot of optimized querying samples at the $90$th iteration. Fig.~(c) illustrates the training loss on MNIST over 20 replications. Finally, Fig.~(d) plots (server's utility, adversary's utility) pairs.
	}\label{sample_path}
	\vspace{-0.1in}
\end{figure*}
The goals of our empirical study are two-fold. First, we show the convergence of the proposed Algo.~\ref{algo1} under several different scenarios. Second, we illustrate the shape of the optimal querying samples obtained from solving Problem (\ref{bilevel}) and depict the Adversary-Benign pairs/curves to demonstrate the relative efficiency of a specific pair of attack and defense. We consider the following three different settings of model extraction attack and defense denoted by \textbf{T1}, \textbf{T2} and \textbf{T3}.

\textbf{Server's Model Types and Defense.}
The server will use: (i) \textbf{(T1)} a linear classifier as its service model and an additive $\sin(\cdot)$ perturbation as its defense strategy, 
(ii) \textbf{(T2)} 
a two-layer feed-forward neural network with soft-max as its service model and another two-layer neural network with different regularization parameters as the defense mechanism, and
(iii) \textbf{(T3)}
a CNN as its service model and a temperature-based soft-max function with tuning parameter $T$ as its defense strategy.

\textbf{Sources of Server's Models.} Server's models are: (i) \textbf{(T1)} synthetically generated,
(ii) \textbf{(T2)} trained on the digits `2' and `8' with MNIST~\cite{lecun1998gradient}, and (iii) \textbf{(T3)} trained on Fashion-MNIST~\cite{xiao2017fashion} with $10$-class.

\textbf{Adversary's Local Model Types.}
The server will use: (i) \textbf{(T1)} logistic regression as its local model 
(ii) \textbf{(T2)} a two-layer neural network as its local model, and (iii) \textbf{(T3)} a two-layer neural network with as its local model.

\textbf{Evaluation Metrics.}
\textbf{T1 \& T2} The loss function in Eq.~(\ref{bilevel}) used for implementing the algorithm is cross-entropy.
On the other hand, the loss function for calculating the server's and adversary's utility (in Def.~\ref{def_ser} and \ref{def_adv}) is $0-1$ loss.
\textbf{T3} The loss function in Eq.~(\ref{bilevel}) used for implementing the algorithm is cross-entropy. 
The loss function for calculating the server's and adversary's utility (in Def.~\ref{def_ser} and \ref{def_adv}) is $\mathcal{L}_{r}(p,q) = \mathbf{1}\{ \| p -q \|^2_2 \geq r \}$ for two probability vectors $p$, $q$ and a positive number $r$.

\vspace{-0.15cm}
\subsection{Experimental results}

\textbf{T1.}
We generate $10^5$ test data  $\underline{x}_t$ uniformly from $[-10,10] \times [-10,10]$, and the labels corresponding to them are $\testy_t = 2*\ind\{\underline{x}_{t,1} \geq \underline{x}_{t,2} -1 \} -1$ for $t=1,\ldots,10000$. 
Let the server's function be $\sgn(f_S(x_i)) = \sgn(x_{i,1} -x_{i,2} + 1)$,  the defense strategy be  $\sgn(\fs^g(x_i)) = \sgn(x_{i,1} + m \sin(\omega x_{i,1} -x_{i,2}+1)$, and $\mathcal{R}$ be $\ell_2$ regularizer.
We set $m = [0.1, 0.5,1,5,10,15]$ and $\omega = [\pm 0.1, \pm 1, \pm 5, \pm 8]$, so there are in total 48 combinations of defense strategies for the server.
The adversary will use the logistic regression, with initial querying samples $x_i \sim \mathcal{N}(0,I_2)$ for $i=1,2,...,200$. 
Regarding the lower-level problem as listed in Line~4 of Algo.~\ref{algo1}, we update the upper-level min-max problem after every 5 iterations of the lower-level problem.

Fig.~\ref{train_loss} demonstrates the adversary's training loss (over 20 replications, solid blue line for mean, shaded region for $\pm 1$ standard deviation) against the iteration. With a diminishing stepsize rule, it converges within 100 iterations.
Fig.~\ref{90th} is a snapshot of (optimized) samples (red triangles) at the $90$th iteration. We noticed that those points are previously lying in the region where labels have been flipped gradually shifted out to the clean region.
In addition, a proportion of optimized samples lay on the decision boundary of the server's defense mechanism, which corroborates the intuition that protecting/attacking the decision boundary is efficient and effective for the server/adversary in Active Learning literature~\cite{settles2009active}.
Moreover, we see that those decision-boundary samples form a symmetric shape and perpendicularly to the true classifier (in black dashed line).

\textbf{T2.} 
We evaluate the proposed algorithm on MINIST~\cite{lecun1998gradient} with a binary classification task on the digits `2' and `8'. 
The initial $100$ query samples and $10^5$ test data are from the uniform distribution.
Fig.~\ref{tsne} demonstrates the prediction performance of the adversary's classifier (over 15 replications, solid red line for mean, shaded region for $\pm 1$ standard deviation). We observed that the prediction performance of the adversary's model gradually improves as the optimization proceeds.
Also, we noted that there are several spikes along the curve. The occurrence of these spikes are mainly due to the fact that the optimal defense strategy with index $\argmax \bm{\lambda} = [\bm{\lambda}_1, \bm{\lambda}_2,...,\bm{\lambda}_m]$ is changing.

By using the optimized querying samples $\overline{\bbx}$ (obtained by running Algo.~\ref{algo1}) as the adversary's attack strategy, we plot several pairs of Adversary-Benign utility in Fig.~\ref{abc}. Almost all the points lie above the $45^{\circ}$ line (blue dashed line), which reflects that the adversary's attack is relatively more efficient than the server's defense for these pairs of attack-defense. Such a case is reasonable since the adversary obtains its attack strategy from solving Eq.~(\ref{bilevel}), which is designed in favor of the adversary by nature.

\textbf{T3.}
We further test the proposed algorithm on Fashion-MNIST. We observed similar results as in the task \textbf{T2} above, and include the details in 
Section~\ref{add_exp} in the appendix.

\vspace{-0.15cm}
\section{Conclusion}\label{conclu}

In this work, we studied the fundamental tradeoffs between the model utility and privacy (in terms of the adversary's utility). The key ingredients of development include a notion of the adversary-benignity curve to evaluate attack-defense pairs, a unified formulation of attack-defense strategies into the min-max bi-level optimization, and an operational algorithm. There are several interesting future problems. One problem is to derive theoretical bounds on the adversary-benignity curve for different loss functions. Also, how to adapt the optimization framework to incorporate specific side information is of interest for further investigation. 
The appendix contains proofs and more experimental studies.

\newpage
\onecolumn


\begin{center}
    {\bf \Large Appendix}
\end{center}

The appendixary document includes the proof of Theorem~\ref{theo1}, an additional case study on the AB Curve along with related technical analysis, two additional experimental results.

\appendix
\section{Proof and Case Study}\label{case_study}
\textbf{Proof of Theorem~\ref{theo1}}
\begin{proof}
By the construction of $f_S^g$, we have 
\begin{align*}
&\mathbb{E}_{x,\tilde{y}}\loss(\tilde{y}, \sgn(f_S(x))) = \int_{\mathcal{X}}\loss(\tilde{y}, \sgn(f_S(x))) dx \\
&= \int_{\mathcal{S}}\loss(\tilde{y}, \sgn(f_S(x)))dx +  \int_{\mathcal{X} \setminus \mathcal{S}}\loss(y, \sgn(f_S(x)))dx \\
&= \frac{1}{2}\int_{\mathcal{S}} \{ \loss(y, \sgn(f_S(x))) + \loss(\ty, \sgn(f_S(x))) \} dx \\
&= \v
\end{align*}

Similarly, from the definition of $f_B$, we have 
\begin{align*}
&\mathbb{E}_{x,\tilde{y}}\loss(\tilde{y}, \sgn(f_B(x))) = \int_{\mathcal{X}}\loss(\tilde{y}, \sgn(f_B(x))) dx \\
&= \int_{\mathcal{S}}\loss(\tilde{y}, \sgn(f_B(x)))dx +  \int_{\mathcal{X} \setminus \mathcal{S}}\loss(y, \sgn(f_B(x)))dx \\
&= \frac{1}{2}\int_{\mathcal{S}} \{ \loss(y, \sgn(f_B(x))) + \loss(\ty, \sgn(f_B(x))) \} dx \\
&= \v \\
&= \mathbb{E}_{x,\tilde{y}}\loss(\tilde{y}, \sgn(\fs(x)))
\end{align*}

Therefore, if the adversary builds its model $\fa$ by using risk minimization algorithm, it cannot distinguish between $f_B$ and $\fs$. By simple calculations, we have $\mathbb{E}_{x,y}\loss(y,\sgn(f_B(x))) = P(\sgn(f_S(x)) \neq \sgn(f_B(x))) = 2\v$.
From the definition of $\v$-learnability, the adversary always needs to set $\fa \leftarrow \fs$, so that it can obtain $P(\sgn(f_S(x)) \neq \sgn(f_A(x))) = 0 < \gamma \in (0,\v)$ with probability one. However, such a selection is not feasible with two indistinguishable options, namely $\fs$ and $f_B$. Hence, we conclude the result.

\end{proof}

\section{An Additional Case Study on the AB Curve}

To further illustrate the AB Curve, we consider the uniformly flipping defense strategy, defined by 
\begin{align}
  \fs^g(x) =
    \begin{cases}
      \fs(x) & \text{with probability $1 - c_x$}\\
      - \fs(x) & \text{with probability $c_x$}
    \end{cases}
\label{unif}
\end{align}
and $\ty = \sgn(\fs^g(x))$.
In other words, the server will flip the (binary) label of its authentic output before responding to the user.
As before, the risk under this type of defense strategy is $\mathbb{E}_{x,\tilde{y}}\mathcal{L}(\tilde{y}, \sgn(\fs(x))).$
And the adversary will obtain its model by solving a risk minimization problem
$\fa = \argmin_{f \in \mathcal{F}_A}\mathbb{E}_{x,\tilde{y}}\mathcal{L}(\tilde{y}, \sgn(f(x))). $
The following results show that \textbf{i)} the adversary can exactly extract the server's model given that server's utility is above $0.5$, and \textbf{ii)} the adversary will reconstruct a completely wrong model if the server's utility is below $0.5$.
\begin{proposition}\label{propos1}
\textbf{i)} If the server uses the strategy defined in Eq.~\ref{unif} with probability $c_x$ (for query $x$) less than $0.5$, then the adversary will exactly extract the model, namely
\begin{align}
    \sup_{x \in \X} c_x < 0.5 \Rightarrow\mathbb{E}_{(x,y)}\mathcal{L}(y,\sgn(\fa(x)))=0. \nonumber
\end{align}
\textbf{ii)}  If the server uses the strategy defined in Eq.~\ref{unif} with probability $c_x$ (for query $x$) greater than $0.5$, then the adversary will obtain zero utility for the reconstructed model $\fa$, namely
\begin{align}
    \inf_{x \in \X} c_x > 0.5 \Rightarrow\mathbb{E}_{(x,y)}\mathcal{L}(y,\sgn(\fa(x)))=1. \nonumber
\end{align}

\end{proposition}

The above result indicates that the adversary `wins' the game if the server maintains a utility level above $0.5$, since the adversary can build a model $\fa$ with perfect accuracy (compared with server's model $\fs$). The reason for such ineffectiveness (for the server) is due to the \textit{Uniform} principle. 
\begin{proof}[Proof of Proposition~\ref{propos1}]
    The first part follows from \cite{manwani2013noise}.
    We prove the second part. 
    By definition, we have $$\mathbb{E}_{x,\tilde{y}}\loss(\tilde{y}, \sgn(f_A(x))) = \int_{\mathcal{X}}\loss(\tilde{y}, \sgn(\fa(x))) dx$$ 
    $$ =\int_{\mathcal{X}}\{(1-c_x) \loss(y, \sgn(f_A(x)))+ c_x \loss(-y, \sgn(\fa(x))) \}dx $$
    $$ = \int_{\{x : \sgn(\fa(x)) \neq \tilde{y}\}}(1-2c_x)\loss(y, \sgn(f_A(x))) dx + \int_{\mathcal{X}} c_x dx.$$

Since $1 - 2c_x < 0$ for all $x \in \mathcal{X}$, in order to minimize $\mathbb{E}_{x,\tilde{y}}\loss(\tilde{y}, \sgn(f_A(x))),$ we must have $\loss(y,\sgn(\fa(x))) = 1$ for all $x \in \mathcal{X}.$ Therefore, we conclude that $ \mathbb{E}_{(x,y)}\mathcal{L}(y,\sgn(\fa(x)))=\int_{\mathcal{X}} 1 dP_x = 1.$
\end{proof}

\vspace{-0.5cm}
\section{Additional Experimental Results}\label{add_exp}
\subsection{Fashion-MNIST Dataset}\label{fasdata}
We further test on Fashion-MNIST~\cite{xiao2017fashion}, and the detailed settings are listed in Experimental Study section in the main paper. Server's is obtained by training on the normal Fashion-MNIST data. 
The initial $100$ query samples and $10^5$ test data are from the uniform distribution.

By setting the optimized querying samples $X$ (obtained by running Algo.~\ref{algo1}) as the adversary's attack, we plot several pairs of Adversary-Benign in Fig.~\ref{fba} Almost all the points lie above the $45^{\circ}$ line (blue dashed line) meaning that the adversary's attack is relatively more efficient than the server's defense for these pairs of attack-defense. Such a case is reasonable since the adversary obtains its attack strategy from solving Eq.(\ref{bilevel}) which is designed in favor for the adversary by nature.

\begin{figure}[!htb]
    \centering
    \includegraphics[width = 6cm, height = 6cm]{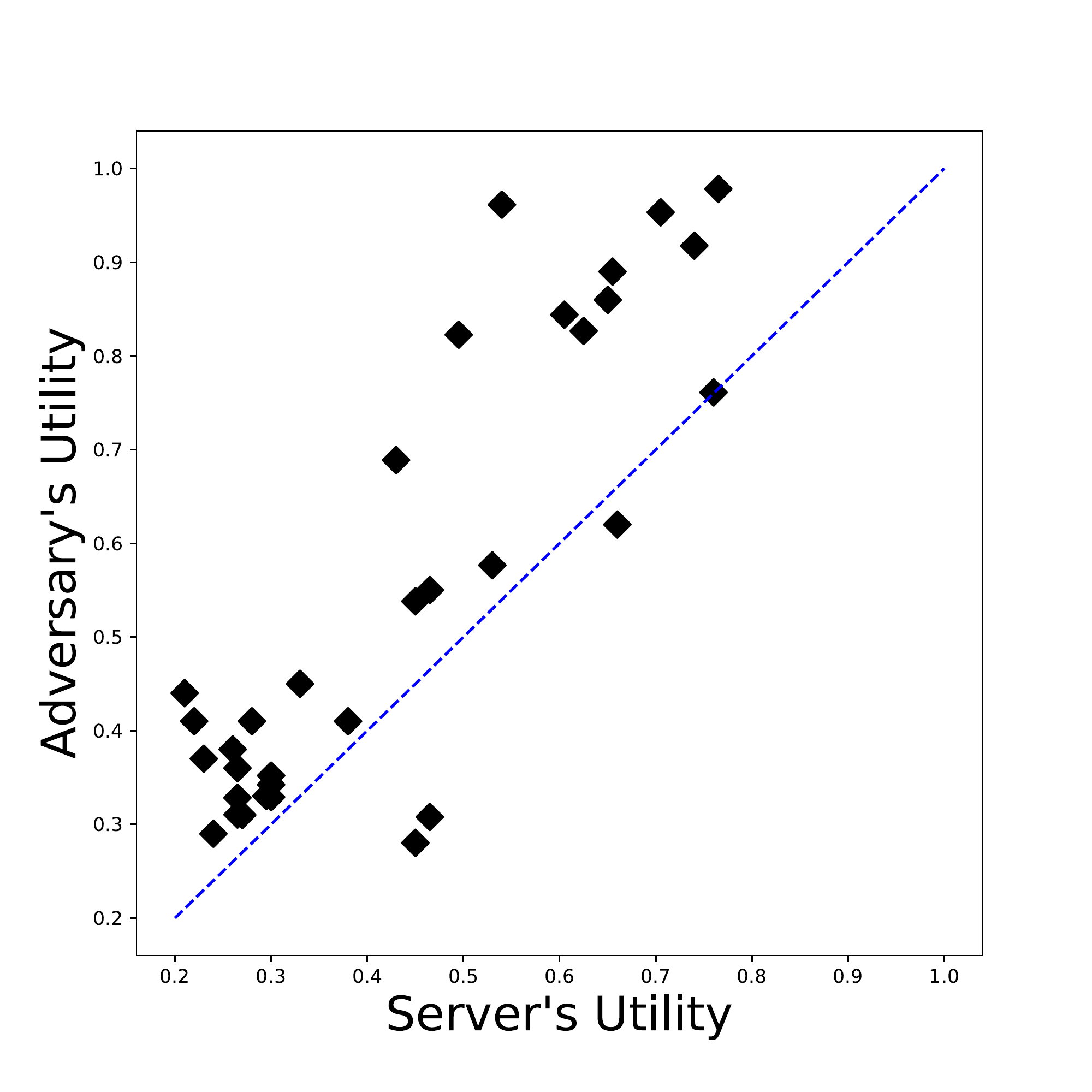}
    \caption{Plot of AB (Adversary-Benign) pairs.}
    \label{fba}
    \vspace{-0.5cm}
\end{figure}

\subsection{Selection over the Defense Strategy Set}

Following the setting as described in the Experimental Study Section of the main paper,  
we demonstrate the server's selection (represented by $\bm{\lambda}$) over the defense strategy set. 
Recall that $\bm{\lambda}$ is the probability simplex that appears in the upper-level maximization problem.
Fig.~\ref{selection_process} below illustrates the evolving process of $\bm{\lambda}$ against iterations. 
At the $0$th iteration, the server puts the folded normal distribution over $\bm{\lambda}$. After one iteration on the upper-level min-max problem, the weights for most strategies shrink to 0, and the strategy with index 25 is the most favorable one. As the algorithm further proceeds, the strategy with index 42 stands out, and eventually, it gains weight 1. The final strategy selected by the server is $(m,\omega) = (15,0.1)$. Such a result is reasonable since, without constraint on the benign utility, the server is likely to choose the one that flips the most querying samples to protect its model.

\begin{figure*}[h]
	\centering
	\subfloat[][ Weight over strategy set at  0th iteration ]{\includegraphics[width=4.6cm]{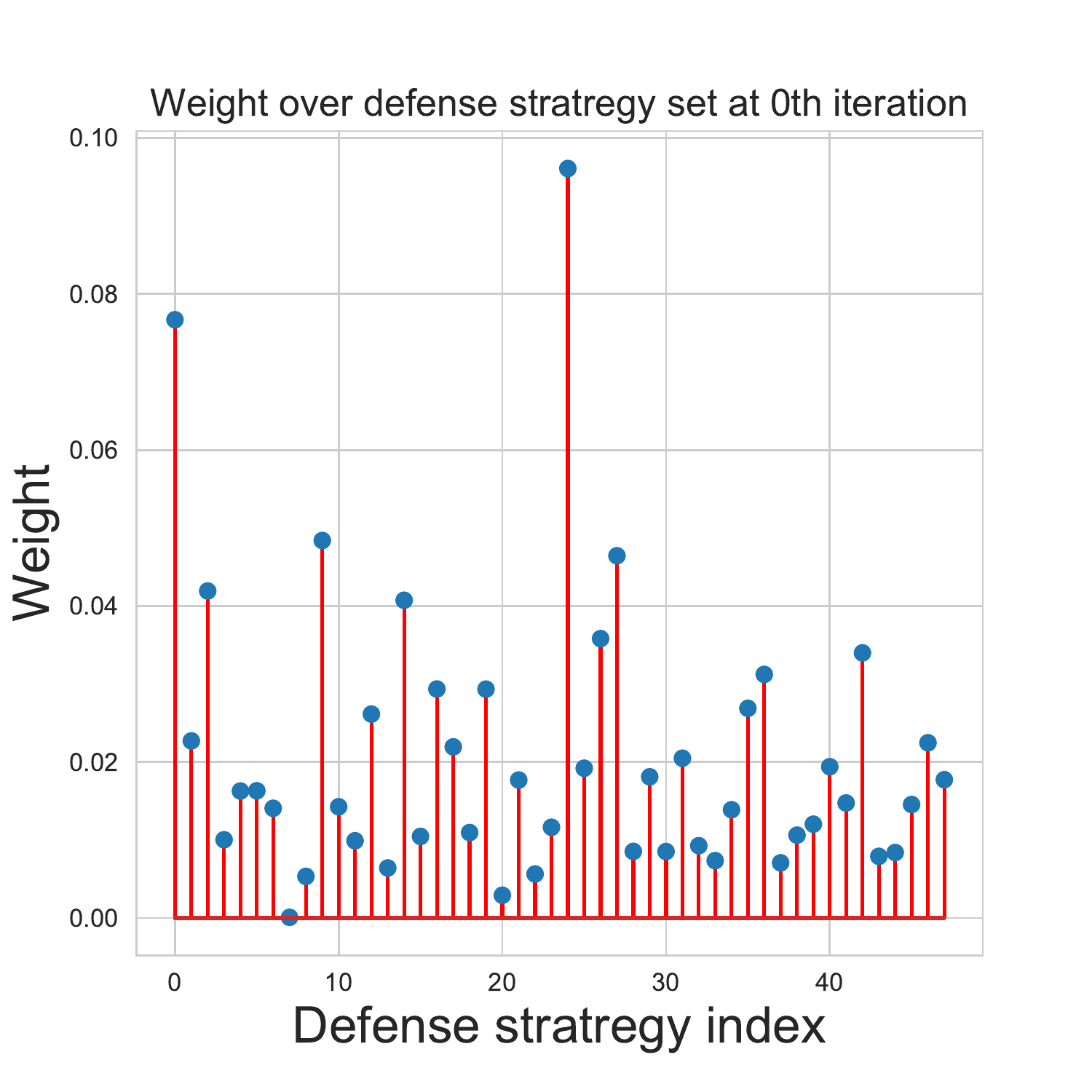}}\hfil  
	\subfloat[][ Weight over strategy set at  1th iteration]{\includegraphics[width=4.6cm]{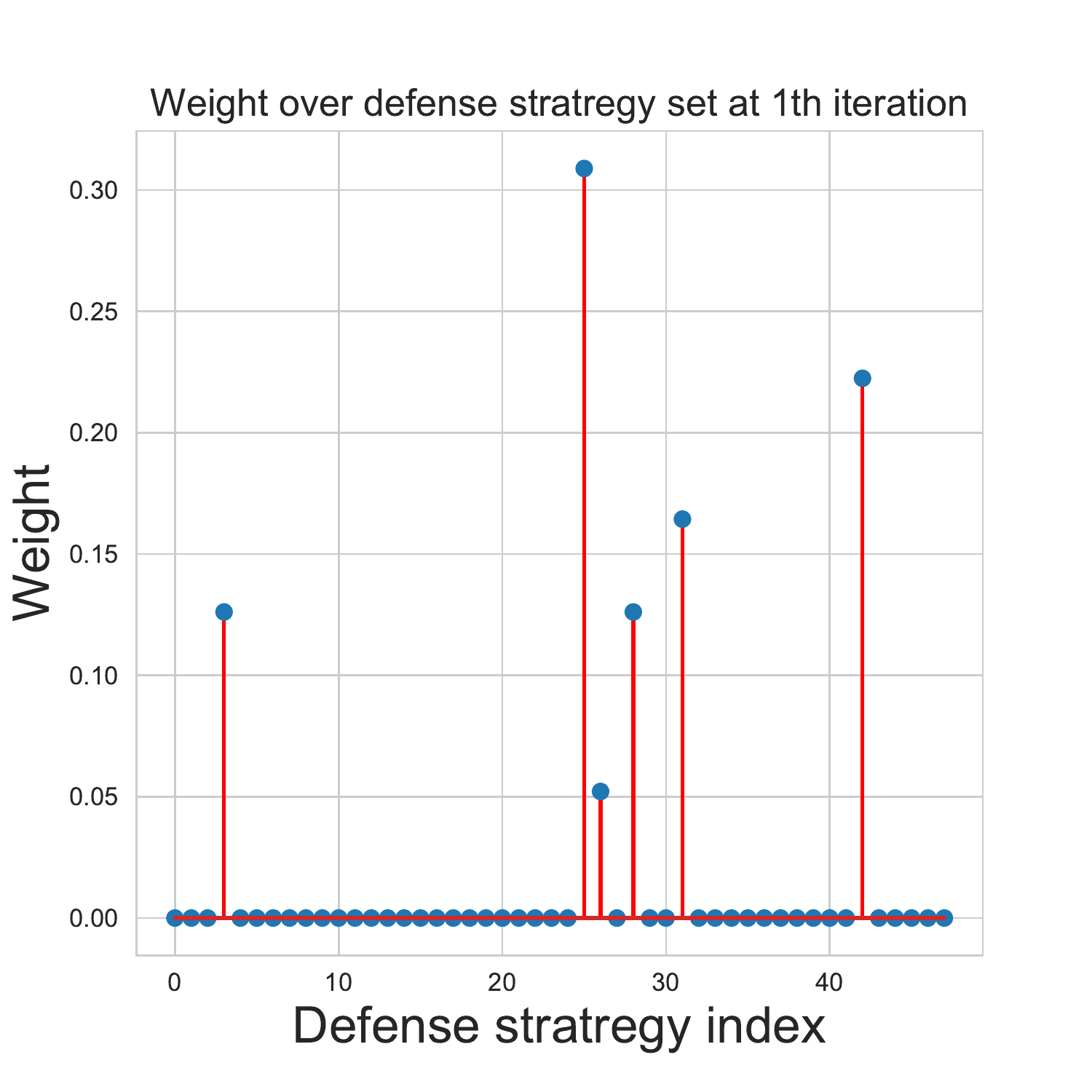}}\hfil
		\subfloat[][ Weight over strategy set at  2th iteration]{\includegraphics[width=4.6cm]{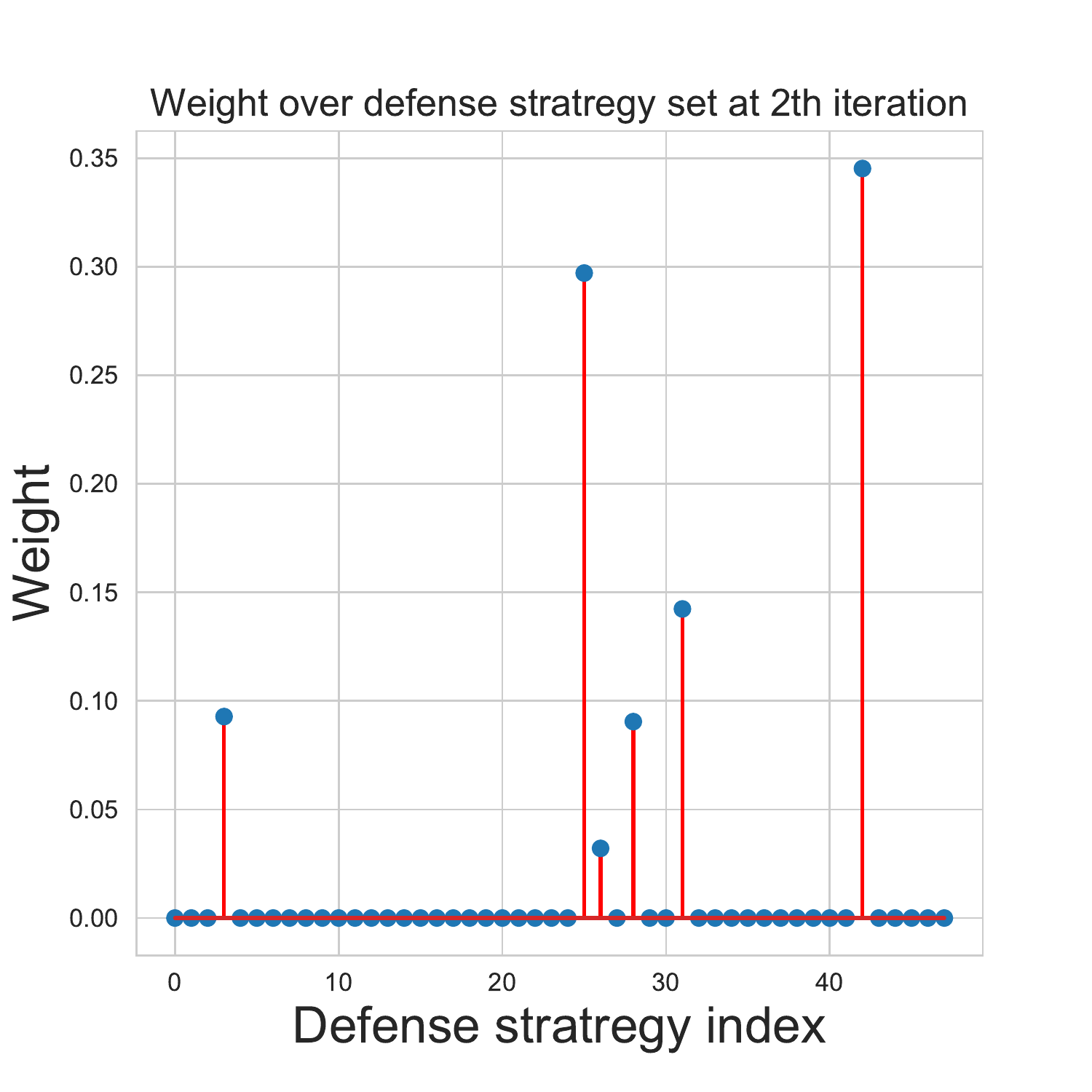}}\hfil
		\subfloat[][ Weight over strategy set at  5th iteration ]{\includegraphics[width=4.6cm]{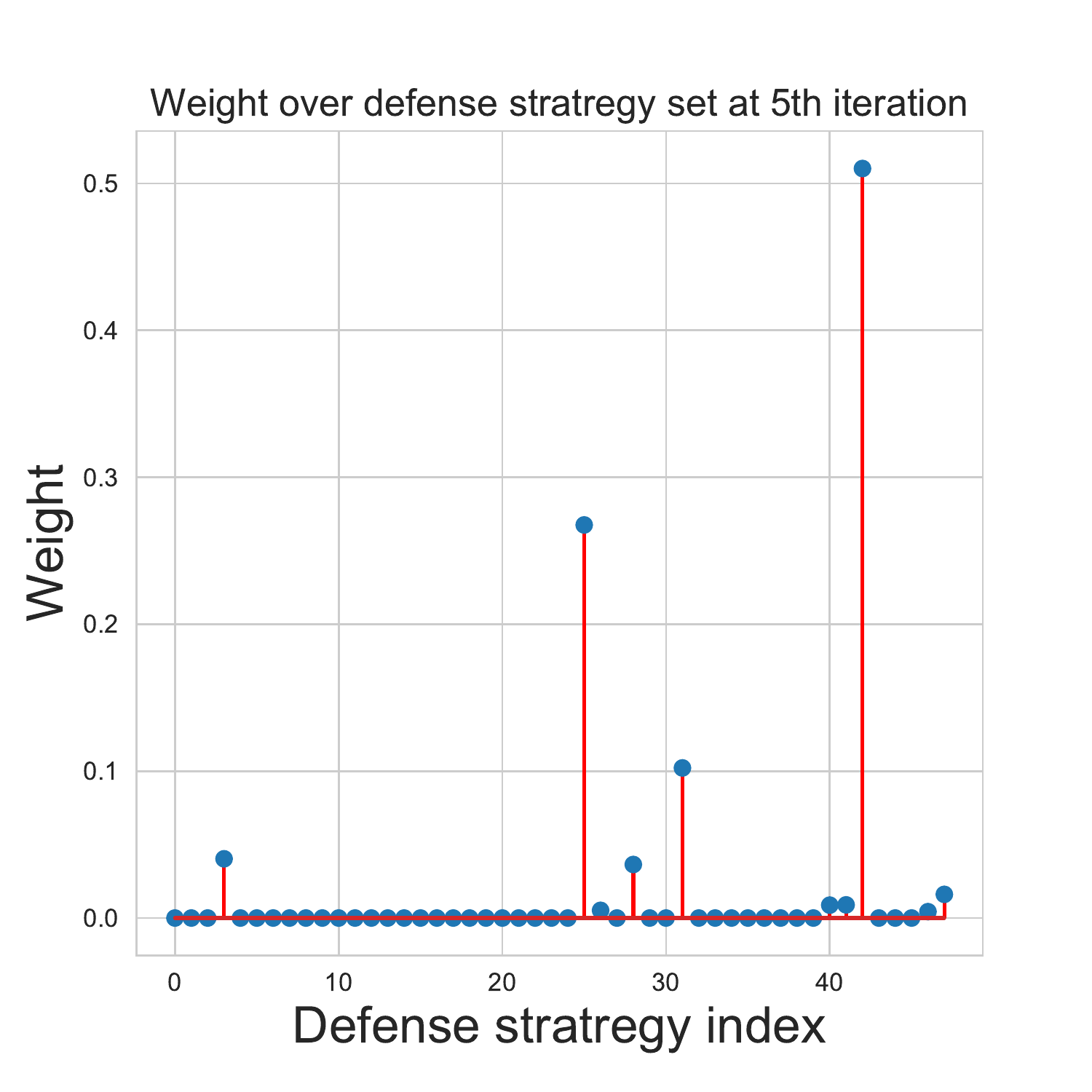}}\hfil  
	\subfloat[][ Weight over strategy set at 10th iteration ]{\includegraphics[width=4.6cm]{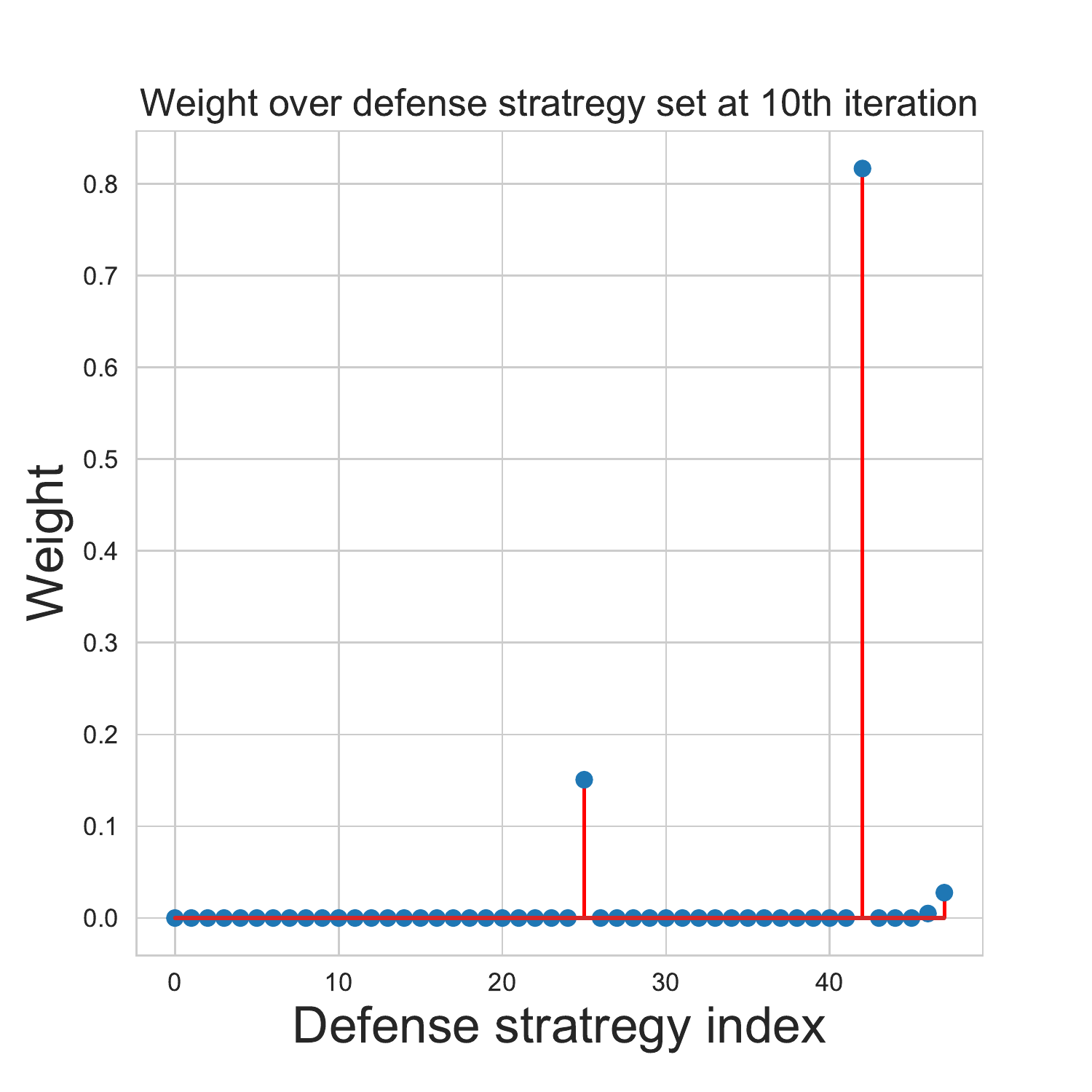}}\hfil
		\subfloat[][ Weight over strategy set at  20th iteration]{\includegraphics[width=4.6cm]{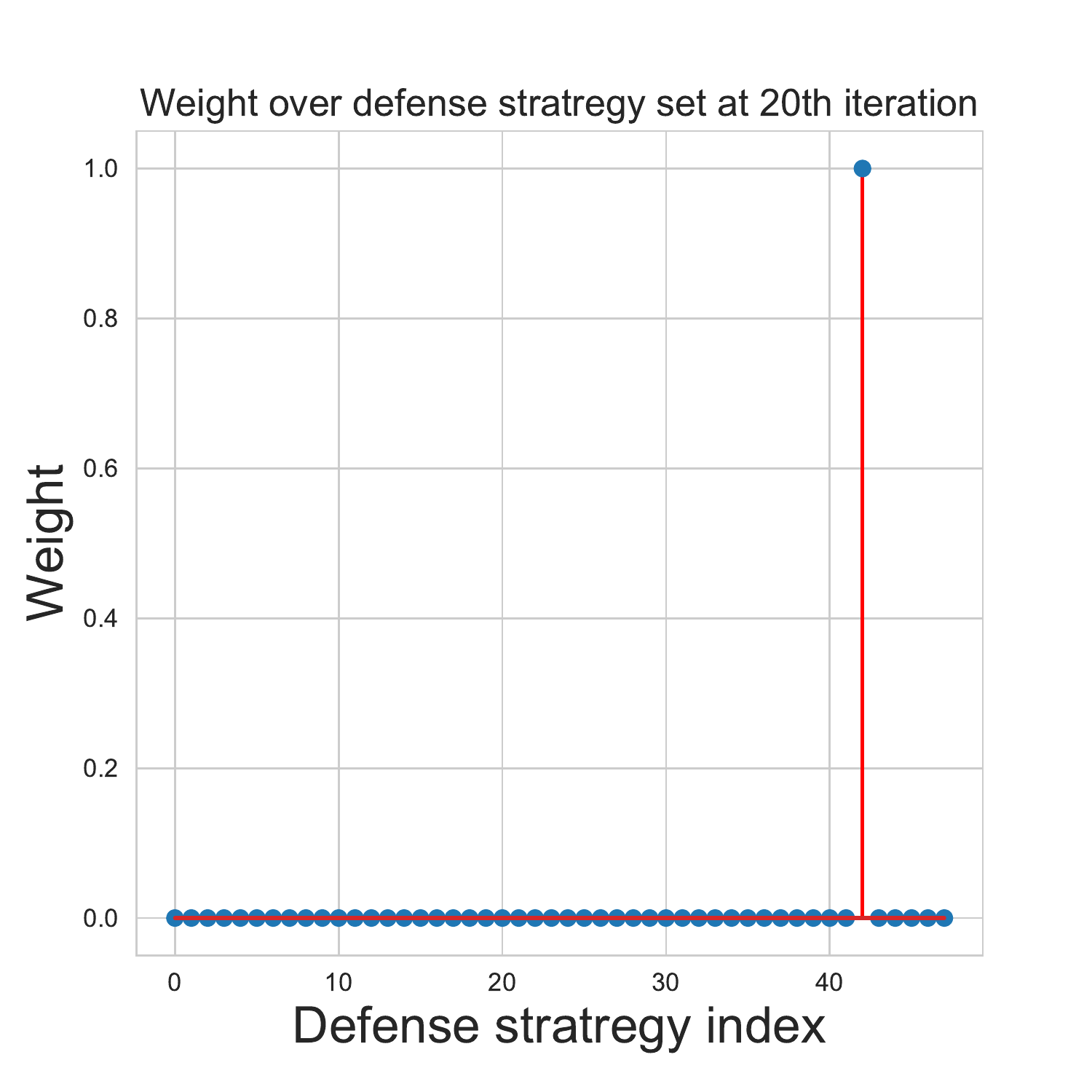}}
	\caption{
	The changing process of $\bm{\lambda}$ against iteration.
	\label{selection_process}}
\end{figure*}

\unappendix
\newpage
\balance
\bibliographystyle{IEEEbib}
\bibliography{here, privacy}

\end{document}